\documentclass{article}

\usepackage{fullpage}
\usepackage[utf8]{inputenc}
\usepackage[english]{babel}
\usepackage{amsmath}
\usepackage{algorithm}
\usepackage{algpseudocode}
\usepackage{fullpage}

\algtext*{EndWhile}% Remove "end while" text
\algtext*{EndIf}% Remove "end if" text
\algtext*{EndFor}% Remove "end if" text

\DeclareMathOperator*{\argmin}{argmin}
\newcommand{\D}{\mathcal{D}}
\newcommand{\cS}{\mathcal{S}}
\newcommand{\boldcS}{\boldsymbol{\cS}}
\newcommand{\cP}{\mathcal{P}}
\newcommand{\cC}{\mathcal{C}}
\newcommand{\I}{\mathcal{I}}
\newcommand{\Imat}{\mathbb{I}}
\newcommand{\R}{\mathcal{R}}

\newcommand{\Z}{\mathcal{Z}}
\newcommand{\A}{\mathcal{A}}
\newcommand{\U}{\mathcal{U}}
\newcommand{\T}{\mathcal{T}}
\newcommand{\bigo}{\mathcal{O}}
\newcommand{\bigotilde}{\tilde{\mathcal{O}}}
\newcommand{\N}{\mathcal{N}}

\newcommand{\Real}{\mathbb{R}}

\newcommand{\thetastar}{\theta^*}
\newcommand{\thetastarr}{\theta^{*r}}
\newcommand{\thetatilde}{\tilde{\theta}}
\newcommand{\thetahat}{\hat{\theta}}
\newcommand{\boldthetahat}{\boldsymbol{\thetahat}}
\newcommand{\boldtheta}{\boldsymbol{\theta}}
\newcommand{\thetahatiavg}{\thetahat_{i,\text{avg}}}
\newcommand{\thetastariavg}{\thetastar_{i,\text{avg}}}

\newcommand{\prob}[1]{\Pr\left[#1\right]}
\renewcommand{\P}{\mathcal{P}}
\newcommand{\proj}{\text{Proj}}
\newcommand{\publish}{f_{\text{publish}}}

\newcommand{\vertiii}[1]{{\left\vert\kern-0.25ex\left\vert\kern-0.25ex\left\vert #1
    \right\vert\kern-0.25ex\right\vert\kern-0.25ex\right\vert}}
\usepackage{authblk}
\usepackage{comment}
\usepackage{hyperref}
\usepackage{amssymb}
\usepackage{amsfonts}
\usepackage{amsmath}
\usepackage{graphicx}
\usepackage{amsthm}
\newtheorem{theorem}{Theorem}[section]
\newtheorem{lemma}[theorem]{Lemma}
\newtheorem{fact}[theorem]{Fact}
\newtheorem{claim}[theorem]{Claim}
\newtheorem{remark}[theorem]{Remark}

\newtheorem{assumption}[theorem]{Assumption}
\newtheorem{definition}[theorem]{Definition}

\usepackage{xcolor}
\usepackage{multirow}
\usepackage{makecell}
\usepackage{dsfont}

\newcommand\shortversion[1]{}
\newcommand\longversion[1]{#1}

\title{Descent-to-Delete: \\ Gradient-Based Methods for Machine Unlearning}
\author{Seth Neel, Aaron Roth, Saeed Sharifi-Malvajerdi}
\affil{University of Pennsylvania}

\date{\today}

\begin{document}

\maketitle
\begin{abstract}
We study the data deletion problem for convex models. By leveraging techniques from convex optimization and reservoir sampling, we give the first data deletion algorithms that are able to handle an arbitrarily long sequence of adversarial updates while promising both per-deletion run-time and steady-state error that do not grow with the length of the update sequence. We also introduce several new conceptual distinctions: for example, we can ask that after a deletion, the entire state maintained by the optimization algorithm is statistically indistinguishable from the state that would have resulted had we retrained, or we can ask for the weaker condition that only the \emph{observable output} is statistically indistinguishable from the observable output that would have resulted from retraining. We are able to give more  efficient deletion algorithms under this weaker deletion criterion.  
\end{abstract}
\longversion{
\clearpage
\tableofcontents
\clearpage
}

\section{Introduction}
%!TEX root = main-full.tex
Users voluntarily provide huge amounts of personal data to online services, such as Facebook, Google, and Amazon, in exchange for useful services. But a basic principle of data autonomy asserts that users should be able to revoke access to their data if they no longer find the exchange of data for services worthwhile. Indeed, each of these organizations provides a way for users to request that their data be deleted. This is related to, although distinct from the ``Right to be Forgotten'' from the European Union's General Data Protection Act (GPDR). The Right to be Forgotten entails the right for users, in certain circumstances, to request that negative information  \emph{concerning} them to be removed. Like basic data autonomy, it sometimes obligates companies to delete data.

But what does it mean to delete data? Typically, user data does not sit siloed in a database, but rather is used to produce derivatives such as predictive models. Deleting a user's data from a database may prevent it from influencing the training of future models\longversion{\footnote{Or perhaps not, if previously trained models (trained before a user's data deletion) are used as inputs to the subsequent models.}}, but does not remove the influence of a user's data on existing models --- and that influence may be significant. For example, it is possible to extract information about specific data points used for training from models that have been trained in standard ways \cite{attack}. So deleting a user's data naively, by simply removing it from a database, may not accomplish much: what we really want is to remove (or at least rigorously limit) the \emph{influence} that an individual's data has on the behavior of any part of the system.

How should we accomplish this? We could \emph{retrain} all predictive models from scratch every time a user requests that their data be removed, but this would entail an enormous computational cost. Ginart et al. \cite{forgetu} propose a compelling alternative: full retraining is unnecessary if we can design a deletion operation that produces a (distribution of) model output(s) that is statistically indistinguishable from the (distribution of) model output(s) that would have arisen from full retraining.  Ginart et al. \cite{forgetu} also propose an approximate notion of deletion that uses a differential-privacy like measure of ``approximate'' statistical indistinguishability that we adopt in this work.

\subsection{Our Results and Techniques}
In this paper, we consider \emph{convex} models that are trained to some specified accuracy, and then are deployed while a sequence of requests arrive to delete (or add) additional data points. The deletion or addition must happen immediately, before the next point comes in, using only a fixed running time (which we measure in terms of gradient computations) per update. We require that the distribution on output models be $(\epsilon,\delta)$-indistinguishable from the distribution on output models that would result from full retraining (see Section \ref{sec:model} for the precise definition: this is a notion of approximate statistical indistinguishability from the differential privacy literature). In a departure from prior work, we make the distinction between whether the entire \emph{internal state} of the algorithm must be indistinguishable from full re-training, or whether we only require statistical indistinguishability with respect to the \emph{observable outputs} of the algorithms. If we require indistinguishability with respect to the full internal state, we call these update or \emph{unlearning} algorithms \emph{perfect}. This is similar to the distinction made in the differential privacy literature, which typically only requires indistinguishability for the \emph{outputs} of private algorithms, but which has a strengthening (called \emph{pan privacy} \cite{pan1,pan2}) which also requires that the internal state satisfy statistical indistinguishability. \longversion{We remark that while unlearning algorithms that are allowed to maintain a ``secret state'' that need not satisfy the data deletion notion require additional trust in the \emph{security} of the training system, this is orthogonal to \emph{privacy}. Indeed, \cite{DifferencingPrivacy} show that even without secret state, algorithms satisfying standard deletion guarantees can exacerbate membership inference attacks if the attacker can observe the model both before and after a deletion (because standard deletion guarantees promise nothing about what can be learned about an individual from two model outputs). In contrast, although some of our unlearning algorithms maintain a secret state that does not satisfy the statistical indistinguishability property, our model outputs themselves satisfy $(\epsilon,\delta)$-differential privacy. This in particular prevents membership inference attacks from observers who can observe a small number of output models, so long as they cannot observe the secret state.} All prior work has focused on perfect unlearning.

We introduce another novel distinction between \emph{strong} unlearning algorithms and \emph{weak} unlearning algorithms. For an unlearning algorithm to be \emph{strong}, we require that for a fixed accuracy target, the run-time of the update operation be constant (or at most logarithmic) in the length of the update sequence. A weak unlearning algorithm may have run-time per update (or equivalently, error) that grows polynomially with the length of the update sequence. All prior work has given weak unlearning algorithms.

We give two sets of results. The first, which operates under the most permissive set of assumptions, is a simple family of gradient descent algorithms. After each addition or deletion request, the update algorithm starts from the previous model, and performs a small number of gradient descent updates --- sufficient to guarantee that the model parameter is boundedly close to the \emph{optimal} model parameter in Euclidean distance. It then perturbs the model parameter with Gaussian noise of sufficient magnitude to guarantee $(\epsilon,\delta)$-indistinguishability with respect to anything within a small neighborhood of the optimal model. We prove that this simple approach yields a strong, perfect unlearning algorithm for loss functions that are strongly convex and smooth. Without the strong convexity assumption, we can still derive strong unlearning algorithms, but ones which must maintain secret state. We can further improve our accuracy guarantees if we are willing to settle for weak unlearning algorithms. The per-round computation budget and the achievable steady state accuracy can be smoothly traded off against one another.

Our second algorithm improves over the straightforward approach above (under slightly stronger regularity assumptions) when the data dimension is sufficiently large.  It first takes a bootstrap sample from the underlying dataset, and then randomly partitions it into $K$ parts. The initial training algorithm separately and independently optimizes the loss function on each part, and then averages the parameter vector from each part, before finally releasing the perturbed average. Zhang et al \cite{ZDW12} analyzed this algorithm (absent the final perturbation) and proved accuracy bounds with respect to the underlying distribution (which for us is the dataset from which we draw the bootstrap sample). Our update operation involves first using a variant of reservoir-sampling that maintains the property that the union of the partitions continue to be distributed as independent samples drawn with replacement from our current dataset. We then use the simple gradient based update algorithms from our first set of results to update the parameters \emph{only from the partitions that have been modified by the addition or deletion}. Because each of these partitions contains only a fraction of the dataset, we can use our fixed gradient computation budget to perform more updates. Because we have maintained the marginal distributions on partition elements via our reservoir sampling step, the overall accuracy analysis of \cite{ZDW12} carries over even after an arbitrary sequence of updates. This is also crucial for our statistical indistinguishability guarantee. The result is a strong unlearning algorithm that yields an improved tradeoff between per-round run-time and steady state accuracy for sufficiently high dimensional data.

\subsection{Related Work}

\iffalse
\begin{table}[h]
\begin{centering}
\begin{tabular}{ |c|c|c|c| }
 \hline
 \multicolumn{4}{|c|}{prior work in machine unlearning} \\
 \hline
 reference & setting & unlearning & secret state?  \\
 \hline
 \makecell{\cite{CY15}} & non-adaptive statistical queries & deterministic &  no \\
\hline
\cite{forgetu} & $k$-means clustering & weak & no \\
\hline
 \cite{hessian} & \makecell{linear regression \\ logistic regression} & weak & no \\
 \hline
\cite{izzo2020approximate} & \makecell{linear regression} & weak & no \\
\hline
\cite{unlearning} & \makecell{deep learning} &  weak &  no \\
\hline
 this work & convex ERM & \makecell{\textbf{strong}}& \makecell{no (strongly convex) \\ yes} \\
 \hline
\end{tabular}
 \end{centering}
 \caption{We give the first strong unlearning algorithms for convex ERM.}
\end{table}
\fi

At a high level, our work differs from prior work in several ways.  We call deletion algorithms that do not maintain secret state \emph{perfect}. All prior work focuses on perfect deletion algorithms, but we give improved bounds for several problems by allowing our algorithms to maintain  secret state.  Second, we allow arbitrary sequences of updates, which can include additions and deletions (rather than just deletions). Finally, we distinguish between weak and strong unlearning algorithms, and give the first strong unlearning algorithms.

Cao and Yang \cite{CY15} first considered the problem of efficiently deleting data from a trained model under a deterministic notion of deletion, and coined the term ``machine unlearning''. They gave efficient deletion methods for certain statistical query algorithms --- but in general, their methods (or indeed, any deterministic notion of deletion) can apply to only very structured problems.  Ginart et al. \cite{forgetu} gave the first definition of data deletion that can apply to randomized algorithms, in terms of statistical indistinguishability. We adopt the approximate deletion notion they introduced, which is itself based on differential privacy \cite{DMNS06,DR14}. Ginart et al. gave a deletion algorithm for the $k$-means problem. Their algorithm is a \emph{weak} deletion algorithm, because their (amortized) running time per update scales linearly with the number of updates.

Guo et al. \cite{hessian} give deletion algorithms for linear and logistic regression, using the same notion of approximate statistical indistinguishability that we use. Their algorithm is similar to our first algorithm: it performs a convex optimization step, followed by a Gaussian perturbation. They use a second order update (a Newton step) rather than first order updates as we do, and their algorithm yields error that grows linearly with the number of updates, and so is a weak deletion algorithm.  Izzo et al \cite{izzo2020approximate} focus on linear regression and show how to improve the run-time per deletion of the algorithm given in \cite{hessian} from quadratic to linear in the dimension.

Our main result leverages a distributed optimization algorithm that partitions the data, separately optimizes on each partition, and then averages the parameters, analyzed by Zhang et al \cite{ZDW12}. Optimizing separately on different partitions of the data, and then aggregating the results is also a well known general technique in differential privacy known as ``Subsample and Aggregate'' \cite{nissim2007smooth} which has found applications in private learning \cite{pate}.  Bourtoule et al. \cite{unlearning} use a similar technique in the context of machine unlearning that they call ``SISA'' (Sharded, Isolated, Sliced, Aggregated). Their goal is more ambitious (to perform deletion for non-convex models), but they have a weaker deletion criterion (that it simply be \emph{possible} that the model arrived at after deletion could have arisen from the retraining process), and they give no error guarantees. Their algorithm involves full retraining on the affected partitions, a different aggregation function, no randomization, and does not include the reservoir sampling step that is crucial to our stronger indistinguishability guarantees.  This distributed optimization algorithm also bears similarity to the well-known \textit{FederatedAveraging} algorithm of \cite{McMahan} used for deep learning in the federated setting.

\longversion{Chen et al. \cite{DifferencingPrivacy} observe that deterministic deletion procedures such as SISA \cite{unlearning} can exacerbate privacy problems when an attacker can observe both the model before and after the deletion of a particular user's data point, and show how to perform membership inference attacks against SISA in this setting. Our method leverages techniques from differential privacy, and so in addition to being an $(\epsilon,\delta)$-deletion algorithm, a view of the two outputs of our algorithm  before and after a deletion is $(2\epsilon,2\delta)$-differentially private, which precludes non-trivial membership inference for reasonable values of $\epsilon$ and $\delta$. This follows because our \emph{deletion} algorithm is randomized: procedures such as the one from \cite{hessian} which have randomized training procedure but deterministic deletion procedure do not share this property.} 

\subsection{Summary of Results}
\label{sec:summary}
%!TEX root = main-full.tex
\begin{table}[t]
%\begin{center}
\centering
\def\arraystretch{1.2}
%\resizebox{\columnwidth}{!}{
%\begin{tabular}{ |p{1.7cm}|p{1.9cm}|p{2cm}|p{2.6cm}|p{2.6cm}|p{2.5cm}| }
\begin{tabular}{|c|c|c|c|c|c|}
 \hline
 \multicolumn{6}{|c|}{tradeoffs for $(\epsilon,\delta)$-unlearning} \\
 \hline
 method & \makecell{loss function \\ properties} & unlearning & accuracy & \makecell{iterations \\ for \\ $i$th update} & baseline iterations \\
 \hline
  \multirow{2}{*}{\makecell{PGD}\rule{0pt}{4ex}}
  & \makecell{SC, smooth} \rule{0pt}{4ex} & \makecell{{strong} \\ (Thm.~\ref{thm:sc-smooth-guarantees})} & $ \frac{ d e^{-\I} }{ \epsilon^2 n^2}$ &  $\I$ & $\I + \log \left( \frac{\epsilon n}{\sqrt{d}} \right) $ \\
  & \makecell{SC, smooth} \rule{0pt}{4ex} & \makecell{{strong, perfect} \\ (\longversion{Thm.~\ref{thm:sc-smooth-guarantees-2}}\shortversion{Supplement})} & $ \frac{ d e^{-\I} }{ \epsilon^2 n^2}$ &  $\makecell{\log i \cdot \I \\ \I \ge \log \left( d / \epsilon \right) }$ & $\I + \log \left( \frac{\epsilon n}{\sqrt{d}} \right)$ \\
 \hline
 \multirow{2}{*}{\makecell{Regularized \\ PGD }\rule{0pt}{4ex}}
 & \makecell{C, smooth} \rule{0pt}{4ex} & \makecell{{strong} \\ (Thm.~\ref{thm:c-smooth-guarantees})} & $\left( \frac{ \sqrt{d} }{ \epsilon n \I } \right)^{\frac{2}{5}}$ & $\I$ & $\left( \frac{ \epsilon n \I }{ \sqrt{d} } \right)^{\frac{2}{5}}$
 \\
& \makecell{C, smooth} \rule{0pt}{4ex} & \makecell{{weak} \\ (\longversion{Thm.~\ref{thm:c-smooth-guarantees-weak}}\shortversion{Supplement})} & $\sqrt{ \frac{ \sqrt{d} }{  \epsilon n \sqrt{\I} } }$ & $i^2 \cdot \I$ & $\sqrt{ \frac{ \epsilon n \sqrt{\I} }{  \sqrt{d} } }$ \\
\hline
\makecell{Distributed \\ PGD}\rule{0pt}{4ex} & \makecell{SC, smooth, \\ Lipschitz \\ and bounded \\ Hessian}\rule{0pt}{4ex} &  \makecell{{strong} \\ (Thm.~\ref{thm:scgd-distributed})} & $\makecell{ \frac{ d e^{-\I n^{\frac{4-3\xi}{2}} } }{ \epsilon^2 n^2}  \\ \ \ \ +  \frac{1}{n^\xi}} $ & $\log i \cdot \I$ & $\makecell{\min \big\{ \log n, \\  \I n^{\frac{4-3\xi}{2}} + \log \left( \frac{\epsilon n}{\sqrt{d}} \right) \big\}}$ \\
 \hline
\end{tabular}
%}
% \end{center}
\caption{Summary of Results. In this table, SC: strongly convex, C: convex, $n$ is training dataset size, $d$ is dimension, $\xi \in [1,4/3]$ is a parameter.}
\label{table}
 \end{table}

In Table \ref{table}, we state bounds for all our unlearning algorithms, and (in the 2nd column) the assumptions that they require (convexity, strong convexity, smoothness, etc.) The 3rd column of the table states whether our algorithms are strong or weak update algorithms (i.e. whether or not their runtimes grow polynomially with the length of the update sequence or not). The 4th column states the steady-state accuracy of the algorithm as a function of the desired run time $\I$ of the first update (each algorithm has a budget of $n\I$ gradient computations per update). The 5th column lists the run-time of the $i$'th update. The 6th column measures the run-time of the baseline approach that would \emph{fully retrain} the model after each update, to the accuracy achieved by our algorithms in the 4th column. Most of these guarantees are for algorithms that maintain a secret state. But for strongly convex and smooth functions we can obtain a perfect unlearning algorithm (i.e. one that satisfies the indistinguishability guarantee not just with respect to observable outputs, but with respect to the entire saved state) with the same asymptotic accuracy/runtime tradeoff, so long as the per-update run-time is at least logarithmic in the dimension.  See \longversion{Theorem \ref{thm:sc-smooth-guarantees-2}}\shortversion{the supplement} for details. For non strongly convex functions, our techniques do not appear to be able to give perfect unlearning algorithms for non-trivial parameters; this is an intriguing direction for future work.

Our ``Distributed PGD'' algorithm is somewhat more complex (see Section \ref{sec:distributed}), but has the advantage that it obtains improved accuracy/run-time tradeoffs for sufficiently high dimensional data. It divides the same  gradient computation budget $n\I$ between different numbers of iterations on different parts of the dataset. See Remark~\ref{distwin} for the exact conditions on when it yields an improvement over our simpler algorithms. 

\section{Model and Preliminaries}
\label{sec:model}
%!TEX root = main-full.tex
 We write $\Z$ to denote the data domain. A dataset $\D$ is a multi-set of elements from $\Z$. Datasets can be modified by \emph{updates} which are requests to either add or remove one element from the dataset.
\begin{definition}[Update]
An update $u$ is a pair $(z, \bullet)$ where $z \in \Z$ is a data point and $\bullet \in \T = \{ \mathtt{'add'}, \mathtt{'delete'} \}$ determines the type of the update. An update sequence $\U$ is a sequence $(u_1, u_2, \ldots )$ where $u_i \in \Z \times \T$ for all $i$. Given a dataset $\D$ and an update $u = (z,\bullet)$, the update operation is defined as follows.
\[
\D \circ u \triangleq \begin{cases} \D \cup \{z\} & \text{if} \ \bullet = \mathtt{'add'}
\\
\D \setminus \{z\} & \text{if} \ \bullet = \mathtt{'delete'} \end{cases}
\]
\end{definition}

 We use $\Theta$ to denote the space of models. In our setting, a \emph{learning} or \emph{training} algorithm is a mapping $\A: \Z^* \to \Theta$ that maps datasets to models. An \emph{unlearning} or \emph{update} algorithm for $\A$ is a mapping $\R_\A : \Z^* \times (\Z \times \T) \times \Theta \to \Theta$ that takes as input a dataset accompanied by a single update, and a model, and outputs an updated model. Some of our update algorithms will also take as input auxiliary information, that we elide here but will be clear from context. The output of the unlearning algorithm itself will not be made public: before any model is made public, it must pass through a \emph{publishing} function. A \emph{publishing function} is a mapping $\publish: \Theta \to \Theta$ that maps a (secret) model to the model that will be made publicly available.  Our unlearning guarantee will informally require that there should be no way to distinguish whether the \emph{published} model resulted from full retraining, or an arbitrary sequence of updates via the unlearning algorithm. Depending on whether we demand \emph{perfect} unlearning or not (to be defined shortly), we may save either the (secret) output of the unlearning algorithm as persistent state, or save only the (public) output of the publishing function.

\begin{definition}[$\D_i, \theta_i, \thetahat_i, \thetatilde_i$]\label{def:notation}
Fix any pair $(\A, \R_\A)$ of learning and unlearning algorithms, any publishing function $\publish$, any dataset $\D$, and any update sequence $\U = (u_1, u_2, \ldots)$. We write $\D_0 = \D$ and for any $i \ge 1$, $\D_i = \D_{i-1} \circ u_i$. For any $i \ge 1$, we write $\theta_i$ for the model input to the unlearning algorithm $\R_\A$ on round $i$. We write $\thetahat_0 = \A \left( \D_0 \right)$, and for any $i \ge 1$, $\thetahat_i = \R_\A \left(\D_{i-1}, u_i, \theta_{i} \right)$. For any $i \ge 0$, we define $\thetatilde_i = \publish (\thetahat_i)$. In other words, whenever $\A$, $\R_\A$, $\publish$, $\D$, and $\U$ are clear from context, we write $\{ \D_i \}_{i \ge 0}$ to represent the sequence of updated datasets, $\{ \theta_i \}_{i \ge 1}$ for the sequence of input models to $\R_\A$, $\{ \thetahat_i \}_{i \ge 0}$ to denote the (secret) output models of $\A$ and $\R_\A$, and $\{ \thetatilde_i \}_{i\ge0}$ to denote their corresponding sequence of published models.
\end{definition}

Our $(\epsilon,\delta)$-unlearning notion is similar to the deletion notion proposed in \cite{forgetu} but generalizes it to an update sequence consisting of both additions and deletions.

\begin{definition}[$(\epsilon,\delta)$-indistinguishability]
Let $X$ and $Y$ be random variables over some domain $\Omega$. We say $X$ and $Y$ are $(\epsilon,\delta)$-indistinguishable and write $X \overset{\epsilon,\delta}{\approx} Y$, if for all $S \subseteq \Omega$,
\[
\prob{X \in S} \le e^{\epsilon} \prob{Y \in S} + \delta, \quad \prob{Y \in S} \le e^{\epsilon} \prob{X \in S} + \delta
\]
\end{definition}

\begin{definition}[$(\epsilon,\delta)$-unlearning]
We say that $\R_\A$ is an $(\epsilon,\delta)$-unlearning algorithm for $\A$ with respect to a publishing function $\publish$, if for all data sets $\D$ and all  update sequences $\U = (u_i)_i$, the following condition holds. For every update step $i \ge 1$, for $\theta_i = \thetahat_{i-1}$
\[
\publish \left( \R_\A \left(\D_{i-1}, u_i, \theta_{i} \right) \right) \overset{\epsilon,\delta}{\approx} \publish \left( \A \left(\D_i \right) \right)
\]
If the above condition holds for $\theta_i = \thetatilde_{i-1}$,  $\R_\A$ is an $(\epsilon,\delta)$-perfect unlearning algorithm for $\A$.
\end{definition}
\longversion{
\begin{remark}
Observe that an unlearning algorithm takes as input the model output by the previous round's unlearning algorithm, whereas a \emph{perfect} unlearning algorithm takes as input the model output by the previous round's publishing algorithm. Since we require that the \emph{published} outputs satisfy $(\epsilon,\delta)$-indistinguishability, this means that unlearning algorithms may need to maintain \emph{secret state} that does not satisfy the indistinguishability guarantee, but that perfect unlearning algorithms do not need to.
\end{remark}

\begin{assumption}\label{ass:ass1}
For notational simplicity (so that we can state asymptotic bounds in terms of $n$) We assume throughout that over the course of an update sequence, the size of the updated datasets never drops below $n/2$ where $n$ is the size of the original training dataset: $\forall \, i, \, n_i \ge n/2$ where $n_i$ is the size of $\D_i$. Note that this is consistent with update sequences being of arbitrary length, since we allow additions as well as deletions. This assumption is not necessary, but otherwise bounds would have to be stated in terms of $n_i$.
\end{assumption}
}

\subsection{Learning Framework: ERM}
We consider an \emph{Empirical Risk Minimization (ERM)} setting in this paper where models are (parameter) vectors in $d$-dimensional  space $\Real^d$ equipped with the (Euclidean) $\ell_2$-norm which will be denoted by $\left\Vert \cdot \right\Vert_2$. Let $\Theta \subseteq \Real^d$ be a convex and closed subset of $\Real^d$, and let $D=\sup_{\theta,\theta' \in \Theta} \Vert \theta - \theta' \Vert_2$ be the \emph{diameter} of $\Theta$. We denote a loss function by a mapping $f: \Theta \times \Z \to \Real$ that takes as input a parameter $\theta \in \Theta$ and a data point $z \in \Z$, and outputs the loss of $\theta$ on $z$, $f (\theta, z)$ --- which we may also denote by $f_z (\theta)$. Given a dataset $\D = \{z_i\}_{i =1}^n \in \Z^n$, with slight abuse of notation, let $f_\D(\theta )$ denote the empirical loss of $\theta$ on the dataset $\D$. In other words,
\begin{equation}\label{eq:empirical}
f_\D \left(\theta\right) \triangleq \frac{1}{n} \sum_{i=1}^{n} f_{z_i} (\theta)
\end{equation}

\begin{definition}[$(\alpha, \beta)$-accuracy]
We say a pair $(\A, \R_\A)$ of learning and unlearning algorithms is $(\alpha, \beta)$-accurate with respect to a publishing function $\publish$, if for every dataset $\D$ and every update sequence $\U$, the following condition holds. For every $i \ge 0$,
\[
\prob{f_{\D_i} (\thetatilde_i) - \min_{\theta \in \Theta} f_{\D_i} (\theta)  > \alpha} < \beta
\]
\end{definition}

\begin{definition}[strong vs. weak unlearning]
Fix any pair $(\A, \R_\A)$ of learning and unlearning algorithms that satisfy $(\alpha, \beta)$-accuracy with respect to some publishing function $\publish$. Let $C_i$ represent the overall computational cost of the unlearning algorithm at step $i$ of the update. We say $\R_\A$ is a ``strong" unlearning algorithm for $\A$ if
\begin{enumerate}
\item $\alpha$ and $\beta$ are independent of the length of the update sequence, and
\item For every $i \ge 1$, $C_i/C_1 = \bigo \left( \log (i) \right)$, i.e., the computation cost of the unlearning algorithm must grow at most logarithmically with $i$.
\end{enumerate}
If (1) holds and $\forall i \ge 1$, $C_i/C_1 =  \Omega \left( poly(i) \right)$, we say $\R_\A$ is a ``weak" unlearning algorithm for $\A$.
\end{definition}
We remark that we have defined update sequences as if they are \emph{non-adaptively chosen}, but that our basic algorithms in Section \ref{sec:basic} have guarantees also for adaptively chosen update sequences. 
\shortversion{In the supplement we have additional preliminaries related to properties of convex functions and gradient descent.}

\longversion{
\subsection{Loss Function Properties}
\begin{definition}[Strong Convexity]\label{def:sc}
A function $h:\Theta \to \Real$ is said to be $m$-strongly convex for some $m \ge 0$, if for any $\theta_1,\theta_2 \in \Theta$, and any $t \in (0,1)$,
\[
h \left( t \theta_1 + (1-t) \theta_2 \right) \le t h (\theta_1) + (1-t) h(\theta_2) - \frac{m}{2} t (1-t) \left\lVert \theta_1 - \theta_2 \right\lVert_2^2
\]
if the above condition holds for $m = 0$, we say $h$ is convex.
\end{definition}

\begin{definition}[Lipschitzness]
A function $h:\Theta \to \Real$ is said to be $L$-Lipschitz if for all $\theta_1,\theta_2 \in \Theta$,
\[
\left\vert h(\theta_1) - h(\theta_2) \right\vert \le L \left\Vert \theta_1 -\theta_2 \right\Vert_2
\]
\end{definition}

\begin{definition}[Smoothness]
A function $h:\Theta \to \Real$ is said to be $M$-smooth, if it is differentiable and for all $\theta_1,\theta_2 \in \Theta$,
\[
\left\Vert \nabla h (\theta_1) - \nabla h (\theta_2) \right\Vert_2 \le M \left\Vert \theta_1 - \theta_2 \right\Vert_2
\]
\end{definition}

\subsection{Strong Convexity and Sensitivity}
Throughout the paper we will leverage the fact that the optimizers of strongly convex functions have low \emph{sensitivity} to individual data points. We will formally state this fact in Lemma~\ref{lem:par} and defer its proof to Appendix~\ref{app:par-proof}.

\begin{lemma}[Sensitivity]
\label{lem:par}
Suppose for any $z \in \Z$, $f_z$ is $L$-Lipschitz and $m$-strongly convex. For any dataset $\D$, let $\thetastar_\D \triangleq \argmin_{\theta \in \Theta} f_\D \left(\theta \right)$. We have that for any integer $n$, any data set $\D$ of size $n$, and any update $u$, $\lVert \thetastar_{\D}-\thetastar_{\D \circ u} \lVert_2 \leq \frac{2L}{mn}$.
\end{lemma}

\subsection{Convergence Results for Gradient Descent}
We make use of projected gradient descent extensively throughout this paper. Here, we state two convergence results for gradient descent that we will use. A crucial feature of these bounds (and one not shared by all bounds for gradient descent and its variants) is that they improve as a function of how close our initial parameter is to the optimal parameter.

Let $h:\Theta \to \Real$ where $\Theta \subseteq \Real^d$ is convex, closed, and bounded. Our goal is to approximate $\min_{\theta \in \Theta} h(\theta)$. The Gradient Descent (GD) algorithm starts with an initial point $\theta_0 \in \Theta$ and proceeds as follows:
\[
\forall \, t \ge 1: \quad \theta_{t} = \proj_\Theta \left( \theta_{t-1} - \eta_t \nabla h (\theta_{t-1}) \right)
\]
$\proj_\Theta (\theta) = \argmin_{\theta' \in \Theta} \Vert \theta - \theta' \Vert_2$ is a projection onto $\Theta$, and $\eta_t$ is the step size used in round $t$.

%\begin{theorem}[Convex and Lipschitz \cite{washu}]
%\label{thm:c-lip-gd}
%Suppose $f$ is convex and $L$-Lipschitz. We have that after $T$ steps of GD with $\eta_t = \frac{f(\theta_{t-1})-f(\thetastar)}{\lVert \nu_{t-1} \lVert_2^2}$, and letting $\tau \in \argmin_{0 \le t \le T} f(\theta_t)$,
%$$
%f (\theta_\tau) -  f(\thetastar) \leq \frac{L \left\lVert \theta_0 - \thetastar \right\lVert_2}{\sqrt{T}}
%$$
%\end{theorem}
%
%\begin{theorem}[Strongly Convex and Lipschitz \cite{washu}]
%\label{thm:sc-lip-gd}
%Suppose $f$ is $m$-strongly convex and $L$-Lipchitz. We have that after $T$ steps of GD with $\eta_t = \frac{2}{mt}$, and letting $\tau \in \argmin_{0 \le t \le T} f(\theta_t)$,
%$$
%f (\theta_\tau) - f(\thetastar) \leq \frac{2L^2}{mT}
%$$
%which by $m$-strong convexity of $f$ implies:
%$$
%\left\lVert \theta_\tau -\thetastar \right\lVert_2 \leq \frac{2L}{m \sqrt{T}}
%$$
%\end{theorem}

\begin{theorem}[Strongly Convex and Smooth \cite{pton}]
\label{thm:sc-sm-gd}
Let $h$ be $m$-strongly convex and $M$-smooth, and let $\thetastar = \argmin_{\theta \in \Theta} h (\theta)$. We have that after $T$ steps of GD with step size $\eta_t = \frac{2}{m +M}$,
\[
\left\lVert \theta_T - \thetastar \right\lVert_2  \leq \left(\frac{M-m}{M+m}\right)^{T} \left\lVert \theta_0-\thetastar \right\lVert_2
\]
\end{theorem}

\begin{theorem}[Convex and Smooth \cite{washu}]
\label{thm:c-sm-gd}
Let $h$ be convex and $M$-smooth, and let $\thetastar \in \argmin_{\theta \in \Theta} h (\theta)$. We have that after $T$ steps of GD with step size $\eta_t = \frac{1}{M}$,
\[
h (\theta_{T}) - \min_{\theta \in \Theta} h (\theta) \le \frac{M \left\lVert \theta_0 - \thetastar \right\lVert_2^2}{2T}
\]
\end{theorem}
}

\section{Basic Perturbed Gradient Descent}
\label{sec:basic}
%!TEX root = main-full.tex

A key building block for our main result (and a simple and effective deletion scheme in its own right, that requires fewer assumptions than our main result) is \emph{perturbed gradient descent}. The basic idea is as follows, for both the training algorithm and the deletion algorithm: we will perform gradient descent updates until we are guaranteed that we have found a $\thetahat_t$ which is within Euclidean distance $\alpha$ of the optimizer, for some small $\alpha$. Our publishing algorithm $\publish$ adds Gaussian noise scaled as a function of $\alpha$ to every coordinate. This guarantees $(\epsilon,\delta)$-indistinguishability with respect to any other parameter that is within distance $\alpha$ of the optimizer --- and hence between the outcomes of full retraining and updating. Depending on whether we want a perfect deletion algorithm or not, we save either the perturbed or unperturbed parameter as our initialization point for the next update.

Our update algorithm will be the same as our training algorithm --- except that it will be initialized at the learned parameter from the previous round, which will guarantee faster convergence. This is because --- if we allow secret state --- the initialization parameter will be within $\alpha$ of the optimizer before the update, and if $f$ is strongly convex, within $\bigo (\alpha + \frac{1}{mn})$ of the optimal parameter after the update\longversion{ by the sensitivity Lemma~\ref{lem:par}}. If we require a perfect deletion algorithm, we will necessarily need to start further from the optimizer, because our saved state will have been additionally perturbed with Gaussian noise. Here we leverage the fact that gradient descent converges quickly when its initialization point is near the optimal solution.

This algorithm relies crucially on leveraging strong convexity, which guarantees us that updates only change the empirical risk minimizer by a small amounts \emph{in parameter space}. In Section \ref{subsec:smooth-convex} we solve the non-strongly-convex case by adding a strongly convex regularizer. \shortversion{Here we present the update and publishing algorithms. The training algorithm $\A$ (in the supplement) is simply projected gradient descent run for more iterations.}

\longversion{
\begin{algorithm}
	\begin{algorithmic}[1]
	\State Input: dataset $\D$
	\State Initialize $\theta'_0 \in \Theta$
		\For {$t=1,2, \ldots T$}
				\State $\theta'_t = \proj_\Theta \left( \theta'_{t-1}-\eta_t \nabla f_\D (\theta'_{t-1}) \right)$
		\EndFor
		\State Output: $\thetahat_0 = \theta'_T$ \Comment{Secret output}
	\end{algorithmic}
\caption{$\A$: \textbf{Learning} for Perturbed Gradient Descent}
\label{alg:scgd}
\end{algorithm}
}

We parameterize our results by the computational cost of the update operations, and we can trade off run-time for accuracy. We measure computational cost by gradient computations. In this section, we parameterize our strong unlearning algorithms by the number of iterations $\I$ that they run for, which corresponds to a budget of $\approx n\I$ gradient computations per update. For weak unlearning algorithms, this is the number of iterations they run for at their first update.

\begin{algorithm}
	\begin{algorithmic}[1]
	\State Input: dataset $\D_{i-1}$, update $u_i$, model $\theta_{i}$
	\State Update dataset $\D_i = \D_{i-1} \circ u_i$
	\State Initialize $\theta'_0 = \theta_{i}$
		\For {$t=1,2, \ldots T_i$}
				\State $\theta'_t = \proj_\Theta \left( \theta'_{t-1}-\eta_t \nabla f_{\D_i} (\theta'_{t-1}) \right)$
		\EndFor
		\State Output $\thetahat_i = \theta'_{T_i}$ \Comment{Secret output}
	\end{algorithmic}
\caption{$\R_\A$: $i$th \textbf{Unlearning} for Perturbed Gradient Descent }
\label{alg:del_scgd}
\end{algorithm}

\begin{algorithm}
	\begin{algorithmic}[1]
		\State Input: $\thetahat \in \Real^d$
		\State Draw $Z \sim \mathcal{N} \left( 0, \sigma^2 \Imat_d \right)$
		\State Output: $\thetatilde = \thetahat + Z$ \Comment{Public output}
	\end{algorithmic}
\caption{$\publish$: Publishing function}
\label{alg:publish}
\end{algorithm}

%!TEX root = main-full.tex
\subsection{Perturbed GD Analysis: Strongly Convex Loss}
\longversion{In this section we analyze Algorithms~\ref{alg:scgd} and \ref{alg:del_scgd} in the case when $f$ is $m$-strongly convex.}
\shortversion{
\begin{theorem}[Accuracy, Unlearning, and Computation Tradeoffs]\label{thm:sc-smooth-guarantees}
Suppose the loss function $f_z$ is $m$-strongly convex, $L$-Lipschitz, and $M$-smooth. Let $\gamma \triangleq (M-m)/(M+m)$. There are parameters for  learning algorithm $\A$  and  unlearning algorithm $\R_\A$ (Algorithm~\ref{alg:del_scgd}) such that:
\begin{enumerate}
\item Unlearning: $\R_\A$ is an $(\epsilon, \delta)$-strong unlearning algorithm for $\A$ with respect to $\publish$.
\item Accuracy: for any $\beta$, $(\A, \R_\A)$ is $(\alpha,\beta)$-accurate with respect to $\publish$ where
\[
\alpha = \bigo \left( \frac{\gamma^{2\I} d \log \left( 1/\delta \right) \log^2 \left(d/\beta \right) }{\epsilon^2 n^2} \right)
\]
\end{enumerate}
\end{theorem}
}
\longversion{
\begin{theorem}[Accuracy, Unlearning, and Computation Tradeoffs]\label{thm:sc-smooth-guarantees}
Suppose for all $z \in \Z$, the loss function $f_z$ is $m$-strongly convex, $L$-Lipschitz, and $M$-smooth. Define $\gamma \triangleq (M-m)/(M+m)$ and $\eta \triangleq 2/(M+m)$. Let the learning algorithm $\A$ (Algorithm~\ref{alg:scgd}) run with $\eta_t = \eta$ and $T \ge \I + \log ( \frac{Dmn}{2L} ) / \log \left( 1 / \gamma \right)$ where $n$ is the size of the input dataset, and let the unlearning algorithm $\R_\A$ (Algorithm~\ref{alg:del_scgd}) run with input models $\theta_i \equiv \thetahat_{i-1}$ and $\eta_t = \eta$ and $T_i = \I$ iterations, for all $i \ge 1$. Let the unlearning parameters $\epsilon$ and $\delta$ be such that $\epsilon = \bigo \left(\log \left( 1/\delta \right) \right)$, and let
\[
\sigma = \frac{ 4 \sqrt{2} L \gamma^\I }{ mn \left( 1 - \gamma^\I \right) \left(\sqrt{\log \left( 1/\delta \right) + \epsilon} - \sqrt{\log \left( 1/\delta \right)} \right)}
\]
in $\publish$ (Algorithm~\ref{alg:publish}). We have that
\begin{enumerate}
\item Unlearning: $\R_\A$ is a strong $(\epsilon, \delta)$-unlearning algorithm for $\A$ with respect to $\publish$.
\item Accuracy: for any $\beta$, $(\A, \R_\A)$ is $(\alpha,\beta)$-accurate with respect to $\publish$ where
\[
\alpha = \bigo \left( \frac{M L^2 \gamma^{2\I} d \log \left( 1/\delta \right) \log^2 \left(d/\beta \right) }{\left( 1 - \gamma^\I \right)^2 m^2 \epsilon^2 n^2} \right)
\]
\end{enumerate}
\end{theorem}
}
\shortversion{Details and proof can be found in the supplementary material.}
\longversion{
\begin{proof}[Proof of Theorem \ref{thm:sc-smooth-guarantees}]
We first prove the unlearning guarantee. Fix a training dataset $\D$ of size $n$ and an update sequence $\U = (u_i)_i$. Recall from Definition~\ref{def:notation} the notation we use: $\{ \D_i \}_{i\ge0}$ for the sequence of updated datasets according to the update sequence $\U$, $\{ \thetahat_i \}_{i\ge0}$ for the sequence of secret non-noisy parameters, and $\{ \thetatilde_i \}_{i\ge0}$ for the sequence of published noisy parameters. We also use $n_i$ to denote the size of $\D_i$. Note that $n_0 = n$ and that by Assumption~\ref{ass:ass1}, $n_i \ge n/2$ for all $i$. Let $\thetastar_i \triangleq \argmin_{\theta} f_{\D_i} (\theta)$ denote the optimizer of $f_{\D_i}$, for any $i \ge 0$.

We have that for any $i \ge 0$, $\publish (\A \left(\D_i \right) ) \sim \N \left(\mu_i, \sigma^2 \Imat_d \right)$, where it follows by the convergence guarantee of Theorem \ref{thm:sc-sm-gd} that
\begin{align}\label{eq:mu}
\begin{split}
\left\Vert \mu_i - \thetastar_{i} \right\Vert_2 &\le \gamma^{T} \left\Vert \theta'_0 - \thetastar_{i} \right\Vert_2
= \frac{2L \gamma^\I \left\Vert \theta'_0 - \thetastar_{i} \right\Vert_2}{Dmn_i}
\le \frac{4L}{mn} \cdot \gamma^\I
\end{split}
\end{align}

We also have that for any $i \ge 1$, $\publish( \R_\A \left( \D_{i-1}, u_i,  \theta_{i} \right) ) \sim  \N \left(\mu'_i, \sigma^2 \Imat_d \right)$ where
\begin{equation}\label{eq:mu'}
 \left\Vert \mu'_i - \thetastar_{{i}} \right\Vert_2 \le \frac{4L}{mn} \cdot \frac{\gamma^\I}{1 - \gamma^\I}
\end{equation}
We use induction on $i$ to prove this claim. Let's focus on the base case $i=1$. We have that
\begin{align*}
 \left\Vert \mu'_1 - \thetastar_{1} \right\Vert_2 &\le \ \gamma^\I \left\Vert \thetahat_0 - \thetastar_{1} \right\Vert_2 \\
 &\le \gamma^\I \left( \left\Vert \thetahat_0 - \thetastar_{0} \right\Vert_2 + \left\Vert \thetastar_{0} - \thetastar_{1} \right\Vert_2 \right) \\
 &\le \gamma^\I \left(  \frac{4L}{mn} \cdot \frac{\gamma^\I}{1 - \gamma^\I} + \frac{4L}{mn} \right) \\
 &= \frac{4L}{mn} \cdot \frac{\gamma^\I}{1 - \gamma^\I}
\end{align*}
The first inequality follows from Theorem \ref{thm:sc-sm-gd} and the fact that when running Algorithm~\ref{alg:del_scgd} for the first update $i=1$, the initial point $\theta'_0 = \theta_1 \equiv \thetahat_0$ saved by the training algorithm. The second inequality is a simple triangle inequality, and the third follows from Equation~(\ref{eq:mu}) (noting that $\thetahat_0 \equiv \mu_0$) and the sensitivity Lemma~\ref{lem:par}. Let's move on to the induction step of the argument. Suppose Equation~(\ref{eq:mu'}) holds for some $i\ge1$. We will show that it holds for $(i+1)$ as well. We have that
\begin{align*}
 \left\Vert \mu'_{i+1} - \thetastar_{{i+1}} \right\Vert_2 &\le \ \gamma^\I \left\Vert \thetahat_i - \thetastar_{{i+1}} \right\Vert_2 \\
 &\le \gamma^\I \left( \left\Vert \thetahat_i - \thetastar_{i} \right\Vert_2 + \left\Vert \thetastar_{i} - \thetastar_{{i+1}} \right\Vert_2 \right) \\
 &\le \gamma^\I \left( \frac{4L}{mn} \cdot \frac{\gamma^\I}{1 - \gamma^\I} + \frac{4L}{mn} \right) \\
 &= \frac{4L}{mn} \cdot \frac{\gamma^\I}{1 - \gamma^\I}
\end{align*}
The first inequality follows from Theorem \ref{thm:sc-sm-gd} and the fact that when running Algorithm~\ref{alg:del_scgd} for the $(i+1)$th update, the initial point $\theta'_0 = \theta_{i+1} = \thetahat_i$ saved by the previous run of the unlearning algorithm. The second inequality is a simple triangle inequality, and the third follows from the induction assumption for $i$ (noting that $\thetahat_i \equiv \mu'_i$), the sensitivity Lemma~\ref{lem:par}, and the assumption that $n_i \ge n/2$.

We therefore have shown that for any $i \geq 1$, for $\theta_i \equiv \thetahat_{i-1}$
\[
\publish \left(\A \left(\D_i \right) \right) \sim \N \left(\mu_i, \sigma^2 \Imat_d \right), \quad \publish \left( \R_\A \left( \D_{i-1}, u_i,  \theta_{i} \right) \right) \sim  \N \left(\mu'_i, \sigma^2 \Imat_d \right)
\]
where Equations~(\ref{eq:mu}) and (\ref{eq:mu'}) imply
\[
\left\Vert \mu_i - \mu'_i  \right\Vert_2 \le \Delta \triangleq \frac{8L}{mn} \cdot \frac{\gamma^\I}{1 - \gamma^\I}
\]
It follows from Lemma~\ref{lem:zcdp} that $\R_\A$ is a $(\frac{\Delta^2}{2 \sigma^2} + \frac{\Delta}{\sigma}\sqrt{2 \log \left( 1 / \delta \right)}, \delta)$-unlearning algorithm for $\A$, where, with $\sigma$ specified in the theorem statement, we get $(\epsilon,\delta)$-unlearning guarantee.

Now let's prove the accuracy statement of the theorem. We will make use of Equations (\ref{eq:mu}) and (\ref{eq:mu'}) and a Gaussian tail bound (see Lemma~\ref{lem:normal-tail}). Recall that for any $i \ge 0$, the published output $\thetatilde_i = \thetahat_i + Z$, and that $\thetahat_0 \equiv \mu_0$ and $\thetahat_i \equiv \mu'_i$ for $i \ge 1$. We therefore have that, for any $\beta$, and for any update step $i\ge 0$,
\[
\prob{\left\Vert \thetatilde_i -\thetastar_{{i}} \right\Vert_2  \ge  \frac{4L}{mn} \cdot \frac{\gamma^\I}{1 - \gamma^\I} + \sigma \sqrt{2d} \log \left(2d/\beta \right)} \le \beta
\]
The choice of $\sigma$ in the theorem and the fact that for $\epsilon = \bigo \left(\log \left( 1/\delta \right) \right)$, we have $\sqrt{\log \left( 1/\delta \right) + \epsilon} - \sqrt{\log \left( 1/\delta \right)} = \Omega ( \epsilon/ \sqrt{\log \left( 1/\delta \right)} )$, imply that for any $i \ge 0$, with probability at least $1-\beta$,
\begin{equation}\label{eq:distance-smooth}
\left\Vert \thetatilde_i -\thetastar_{{i}} \right\Vert_2 = \bigo \left( \frac{L \gamma^{\I} \sqrt{d \log \left( 1/\delta \right)} \log \left(d/\beta \right) }{\left( 1 - \gamma^\I \right) \epsilon m n } \right)
\end{equation}
Finally, since $f_z$ is $M$-smooth for all $z$, we get that for any update step $i \ge 0$, with probability at least $1-\beta$,
\[
f_{\D_i} ( \thetatilde_i ) - f_{\D_i} ( \thetastar_{i} ) \le \frac{M}{2} \left\Vert \thetatilde_i -\thetastar_{{i}} \right\Vert_2^2 = \bigo \left( \frac{M L^2 \gamma^{2\I} d \log \left( 1/\delta \right) \log^2 \left(d/\beta \right) }{\left( 1 - \gamma^\I \right)^2 m^2  \epsilon^2n^2 } \right)
\]
\end{proof}
}

The same algorithm can be analyzed as a \emph{perfect} unlearning algorithm (i.e. without maintaining secret state). It obtains the same asymptotic tradeoff between running time and accuracy, under the condition that the per-update run-time is at least logarithmic in the relevant parameters. Intuitively, this run-time lower bound is required so that the update algorithm can ``recover'' from the effect of the added noise. \shortversion{The formal statement of this guarantee and its proof can be found in the supplement.}
\longversion{
\begin{theorem}[Perfect Unlearning]\label{thm:sc-smooth-guarantees-2}
Suppose for all $z \in \Z$, the loss function $f_z$ is $m$-strongly convex, $L$-Lipschitz, and $M$-smooth. Define $\gamma \triangleq (M-m)/(M+m)$ and $\eta \triangleq 2/(M+m)$. Let the unlearning parameters $\epsilon$ and $\delta$ be such that $\epsilon = \bigo \left(\log \left( 1/\delta \right) \right)$. Let the learning algorithm $\A$ (Algorithm~\ref{alg:scgd}) run with $\eta_t = \eta$ and $T \ge \I + \log ( \frac{Dmn}{2L} ) / \log \left( 1 / \gamma \right)$ where $n$ is the size of the input dataset, and let the unlearning algorithm $\R_\A$ (Algorithm~\ref{alg:del_scgd}) run with input models $\theta_i \equiv \thetatilde_{i-1}$ and $\eta_t = \eta$ and $T_i = \I +  \log \left(\log \left( 4di/\delta \right) \right) / \log \left( 1 / \gamma \right)$ iterations for all $i \ge 1$ where
\[
\I \ge \frac{ \log \left( \frac{\sqrt{2d} \, (1-\gamma)^{-1}}{ \sqrt{2 \log \left( 2/\delta \right) + \epsilon} - \sqrt{2 \log \left( 2/\delta \right)} } \right) }{\log \left( 1 / \gamma \right)}, \ \text{and} \quad \sigma = \frac{ 8 L \gamma^\I \left( 1 - \gamma^\I \right)^{-1}}{ mn \left(\sqrt{ 2 \log \left( 2/\delta \right) + 3 \epsilon} - \sqrt{2 \log \left( 2/\delta \right) + 2 \epsilon } \right)}
\]
in $\publish$ (Algorithm~\ref{alg:publish}). We have that
\begin{enumerate}
\item Unlearning: $\R_\A$ is a strong $(\epsilon, \delta)$-perfect unlearning for $\A$ with respect to $\publish$.
\item Accuracy: for any $\beta$, $(\A, \R_\A)$ is $(\alpha,\beta + \delta)$-accurate with respect to $\publish$ where
\[
\alpha = \bigo \left( \frac{M L^2 \gamma^{2\I} d \log \left( 1/\delta \right) \log^2 \left(d/\beta \right) }{\left( 1 - \gamma^\I \right)^2 m^2 \epsilon^2 n^2} \right)
\]
\end{enumerate}
\end{theorem}
The proof of this theorem can be found in Appendix~\ref{sec:app-perfect}.
}

\subsection{Convex Loss: Regularized Perturbed GD}\label{subsec:smooth-convex}
If our loss function is not strongly convex, we can regularize it to enforce strong convexity, and apply our algorithms to the regularized loss function. When we do this, we must manage a basic tradeoff: the more aggressively we regularize the loss function, the less sensitive it will be, and so the less noise we will need to add in our $\publish$ routine. This reduced noise will \emph{increase} accuracy. On the other hand, the more aggressively we regularize, the less well the optimizer of the regularized loss function will optimize the original loss function of interest, which will \emph{decrease} accuracy. More aggressive regularization will also degrade the Lipschitz and smoothness guarantees of the loss function. We choose our regularization parameter carefully to trade off these various sources of error. \shortversion{All together, we get the following theorem: details and proof are in the supplement.}

\longversion{
Suppose in this section, without loss of generality, that $\Theta$ contains the origin: $0 \in \Theta$. This will imply that $\sup_{\theta \in \Theta} \Vert \theta \Vert_2 \le D$ where $D$ is the diameter of $\Theta$, as before. Our strategy is to regularize $f$ so as to make it strongly convex, and have our learning and unlearning algorithms run on the regularized version of $f$. Let, for any $z \in \Z$ and any $\theta \in \Theta$, for some $m > 0$,
\begin{equation}\label{eq:regularizer}
g_z (\theta) \triangleq f_z ( \theta ) + \frac{m}{2} \left\Vert \theta \right\Vert_2^2
\end{equation}
\begin{claim}
If $f_z$ is convex, $L$-Lipschitz, and $M$-smooth, then $g_z$ is $m$-strongly convex, $(L+mD)$-Lipschitz, and $(M+m)$-smooth.
\end{claim}
}

\shortversion{
\begin{theorem}[Accuracy, Unlearning, and Computation Tradeoffs]\label{thm:c-smooth-guarantees}
Suppose the loss function $f_z$ is convex, $L$-Lipschitz, and $M$-smooth. There are parameters such that the learning algorithm $\A$ and the unlearning algorithm $\R_\A$ (Algorithm~\ref{alg:del_scgd}) run on the regularized loss function satisfy:
\begin{enumerate}
\item Unlearning: $\R_\A$ is an $(\epsilon, \delta)$-strong-unlearning algorithm for $\A$ with respect to $\publish$.
\item Accuracy: for any $\beta$, $(\A, \R_\A)$ is $(\alpha,\beta)$-accurate with respect to $\publish$ where
$$
\alpha  = \bigo \left( \left( \frac{\sqrt{d \log \left( 1/\delta \right)}}{\epsilon n \I} \right)^{\frac{2}{5}}  \log^2 \left( d / \beta \right) \right) + \bigo \left( n^{-\frac{4}{5}} \right) + \bigo \left( n^{-\frac{6}{5}} \right)
$$
\end{enumerate}
\end{theorem}
}

\longversion{
\begin{theorem}[Accuracy, Unlearning, and Computation Tradeoffs]\label{thm:c-smooth-guarantees}
Suppose for all $z \in \Z$, the loss function $f_z$ is convex, $L$-Lipschitz, and $M$-smooth, and let $g_z$ be defined as in Equation~(\ref{eq:regularizer}) for some $m$ specified later. Define $\gamma \triangleq M/(M+2m)$ and $\eta \triangleq 2/(M+2m)$. Let the learning algorithm $\A$ (Algorithm~\ref{alg:scgd}) run on the regularized $g$ with $\eta_t = \eta$ and $T \ge \I + \log ( \frac{Dmn}{2L} ) / \log \left( 1 / \gamma \right)$ where $n$ is the size of the input dataset, and let the unlearning algorithm $\R_\A$ (Algorithm~\ref{alg:del_scgd}) run on the regularized $g$ with input models $\theta_i \equiv \thetahat_{i-1}$ and $\eta_t = \eta$ and $T_i = \I$ iterations for all $i \ge 1$. Let the unlearning parameters $\epsilon$ and $\delta$ be such that $\epsilon = \bigo \left(\log \left( 1/\delta \right) \right)$, and let
\[
\sigma = \frac{ 4 \sqrt{2} \left(L+mD \right) \gamma^\I }{ mn \left( 1 - \gamma^\I \right) \left(\sqrt{\log \left( 1/\delta \right) + \epsilon} - \sqrt{\log \left( 1/\delta \right)} \right)}, \ \ m = \left( \frac{L M^{\frac{3}{2}} \sqrt{d \log \left( 1/\delta \right)}}{D \epsilon n \I} \right)^{\frac{2}{5}}
\]
where $\sigma$ is the noise level in $\publish$. We have that
\begin{enumerate}
\item Unlearning: $\R_\A$ is a strong $(\epsilon, \delta)$-unlearning algorithm for $\A$ with respect to $\publish$.
\item Accuracy: for any $\beta$, $(\A, \R_\A)$ is $(\alpha,\beta)$-accurate with respect to $\publish$ where
\[
\alpha  = \bigo \left( \left( \frac{M^{\frac{3}{2}} L D^4 \sqrt{d \log \left( 1/\delta \right)}}{\epsilon n \I} \right)^{\frac{2}{5}}  \log^2 \left( d / \beta \right) \right) + \bigo \left( n^{-\frac{4}{5}} \right) + \bigo \left( n^{-\frac{6}{5}} \right)
\]
\end{enumerate}
\end{theorem}
}

\longversion{
\begin{proof}[Proof of Theorem~\ref{thm:c-smooth-guarantees}]
The unlearning guarantee of the theorem holds for any $m > 0$, and follows from Theorem~\ref{thm:sc-smooth-guarantees} by the choice of $\sigma$ in the theorem statement. Let's prove the accuracy statement. Let $\thetastarr_i = \argmin_{\theta \in \Theta} g_{\D_i} (\theta)$ denote the optimizer of the regularized $g_{\D_{i}}$, for all $i \ge 0$. It follows from the proof of Theorem~\ref{thm:sc-smooth-guarantees} (see Equation~(\ref{eq:distance-smooth})) that for any update step $i \ge 0$, with probability $1-\beta$,
\begin{equation}\label{eq:theta'-opt}
\left\Vert \thetatilde_i - \thetastarr_i \right\Vert_2  = \bigo \left( \frac{\left(L+mD \right) \gamma^{\I} \sqrt{d \log \left( 1/\delta \right)} \log \left(d/\beta \right) }{\left( 1 - \gamma^\I \right) \epsilon m n } \right)
\end{equation}
Also note that
\begin{equation}\label{eq:approx}
\frac{\gamma^\I}{1 - \gamma^\I} = \frac{1}{\left( 1 + 2 \left(m/M \right) \right)^\I - 1} \le \frac{M}{m \I}
\end{equation}
Now let $\thetastar_i \in \argmin_{\theta \in \Theta} f_{\D_i} (\theta)$ denote an optimizer of the original loss function $f_{\D_i}$, for any $i \ge 0$. We have that, for any $i \ge 0$,
\begin{align}\label{eq:convex-analysis}
\begin{split}
f_{\D_i} ( \thetatilde_i ) - f_{\D_i} ( \thetastar_{i} ) &= f_{\D_i} ( \thetatilde_i ) - f_{\D_i} ( \thetastarr_i ) + f_{\D_i} ( \thetastarr_i ) - f_{\D_i} ( \thetastar_{i} )\\
&\overset{(1)}{\le} \nabla  f_{\D_i} ( \thetastarr_i )^\top \left( \thetatilde_i - \thetastarr_i \right) + \frac{M}{2} \left\Vert \thetatilde_i - \thetastarr_i \right\Vert_2^2 + f_{\D_i} ( \thetastarr_i ) - f_{\D_i} ( \thetastar_{i} ) \\
&\overset{(2)}{=} \frac{M}{2} \left\Vert \thetatilde_i - \thetastarr_i \right\Vert_2^2 + m \theta_i^{{*r}\top} \left( \thetastarr_i  -  \thetatilde_i \right) + f_{\D_i} ( \thetastarr_i ) - f_{\D_i} ( \thetastar_{i} ) \\
&\overset{(3)}{\le} \frac{M}{2} \left\Vert \thetatilde_i - \thetastarr_i \right\Vert_2^2 + m D^2 + f_{\D_i} ( \thetastarr_i ) - f_{\D_i} ( \thetastar_{i} ) \\
&= \frac{M}{2} \left\Vert \thetatilde_i - \thetastarr_i \right\Vert_2^2 + m D^2 + g_{\D_i} (\thetastarr_i) - \frac{m}{2} \left\Vert \thetastarr_i \right\Vert_2^2 - f_{\D_i} ( \thetastar_{i} ) \\
&\overset{(4)}{\le} \frac{M}{2} \left\Vert \thetatilde_i - \thetastarr_i \right\Vert_2^2 + m D^2 + g_{\D_i} (\thetastar_{i}) - \frac{m}{2} \left\Vert \thetastarr_i \right\Vert_2^2 - f_{\D_i} ( \thetastar_{i} ) \\
&= \frac{M}{2} \left\Vert \thetatilde_i - \thetastarr_i \right\Vert_2^2 + m D^2   + \frac{m}{2} \left( \left\Vert \thetastar_{i} \right\Vert_2^2 - \left\Vert \thetastarr_i \right\Vert_2^2\right) \\
&\overset{(5)}{=} \bigo \left( \frac{M^3 \left(L+mD \right)^2 d \log \left( 1/\delta \right) \log^2 \left(d/\beta \right) }{m^4 \epsilon^2  n^2 \I^2 } + m D^2 \right)
\end{split}
\end{align}
where inequality (1) follows from $f_{\D_i}$ being $M$-smooth. (2) follows from the fact that for all $\theta$, $\nabla  f_{\D_i} ( \theta ) = \nabla  g_{\D_i} ( \theta ) - m \theta$ and that by optimality of $\thetastarr_i$ for $g_{\D_i}$, we have $\nabla  g_{\D_i} ( \thetastarr_i ) = 0$. (3) follows from a simple application of Cauchy-Schwarz: for all $\theta_1, \theta_2 \in \Theta$, we have $\theta_1^\top \theta_2 \le \Vert \theta_1\Vert_2 \Vert  \theta_2\Vert_2 \le D^2$. (4) follows from the optimality of $\thetastarr_i$ for $g_{\D_i}$, and (5) is implied by Equations~(\ref{eq:theta'-opt}) and (\ref{eq:approx}), and it holds with probability $1-\beta$. Now for the choice of $m$ in the theorem, we conclude that for any $i \ge 0$, with probability $1-\beta$,
\[
f_{\D_i} ( \thetatilde_i ) - f_{\D_i} ( \thetastar_{i} ) = \bigo \left( \left( \frac{M^{\frac{3}{2}} L D^4 \sqrt{d \log \left( 1/\delta \right)}}{\epsilon n \I} \right)^{\frac{2}{5}}  \log^2 \left( d / \beta \right) \right) + \bigo \left( n^{-\frac{4}{5}} \right) + \bigo \left( n^{-\frac{6}{5}} \right)
\]
\end{proof}
}

If our goal is to satisfy only weak unlearning (i.e. to allow our run-time to grow with the length of the update sequence), we can obtain error bounds that have a better dependence on $n$.  \shortversion{Details are in the supplement.}
\longversion{
\begin{theorem}[Accuracy, Unlearning, and Computation Tradeoffs]\label{thm:c-smooth-guarantees-weak}
Suppose for all $z \in \Z$, the loss function $f_z$ is convex, $L$-Lipschitz, and $M$-smooth, and let $g_z$ be defined as in Equation~(\ref{eq:regularizer}) for some $m$ specified later. Define $\gamma \triangleq M/(M+2m)$ and $\eta \triangleq 2/(M+2m)$. Let the learning algorithm $\A$ (Algorithm~\ref{alg:scgd}) run on the regularized $g$ with $\eta_t = \eta$ and $T \ge \I + \log ( \frac{Dmn}{2L} ) / \log \left( 1 / \gamma \right)$ where $n$ is the size of the input dataset, and let the unlearning algorithm $\R_\A$ (Algorithm~\ref{alg:del_scgd}) run on the regularized $g$ with input model $\theta_i \equiv \thetahat_{i-1}$ and $\eta_t = \eta$ and $T_i = i^2 \cdot \I$ iterations, for the $i$th update. Let the unlearning parameters $\epsilon$ and $\delta$ be such that $\epsilon = \bigo \left(\log \left( 1/\delta \right) \right)$, and let
\[
\sigma = \frac{ 2 \sqrt{2M} \left(L+mD \right)}{ m\sqrt{m \I} n \left(\sqrt{\log \left( 1/\delta \right) + \epsilon} - \sqrt{\log \left( 1/\delta \right)} \right)}, \ \ m = \sqrt{ \frac{L M \sqrt{d \log \left( 1/\delta \right)}}{D \epsilon n \sqrt{\I}} }
\]
where $\sigma$ is the noise level in $\publish$. We have that
\begin{enumerate}
\item Unlearning: $\R_\A$ is a weak $(\epsilon, \delta)$-unlearning algorithm for $\A$ with respect to $\publish$.
\item Accuracy: for any $\beta$, $(\A, \R_\A)$ is $(\alpha,\beta)$-accurate with respect to $\publish$ where
\[
\alpha  = \bigo \left(  \sqrt{ \frac{M L D^3 \sqrt{d \log \left( 1/\delta \right)}}{\epsilon n \sqrt{\I}} } \log^2 \left( d / \beta \right) \right) + \bigo \left( n^{-1} \right) + \bigo \left( n^{-\frac{3}{2}} \right)
\]
\end{enumerate}
\end{theorem}

\begin{remark}\label{rem:general-tradeoff}
We remark that we can further explore the tradeoff between each update's runtime $T_i$ and dependence on sample size $n$. Let $\xi \ge 1$ be any constant (Theorem~\ref{thm:c-smooth-guarantees-weak} corresponds to $\xi = 1$). We have that under the setting of Theorem~\ref{thm:c-smooth-guarantees-weak}, with $T_i = i^{2\xi} \cdot \I$ iterations, and
\[
\sigma = \frac{ 2 \sqrt{2} M^{\frac{1}{2\xi}} \left(L+mD \right)}{ m (m \I)^{\frac{1}{2 \xi}} n \left(\sqrt{\log \left( 1/\delta \right) + \epsilon} - \sqrt{\log \left( 1/\delta \right)} \right)}, \ \ m = \left( \frac{L^2 M^{\frac{1+\xi}{\xi}} d \log \left( 1/\delta \right)}{D^2 \epsilon^2 n^2 \I^{\frac{1}{\xi}}} \right)^{\frac{\xi}{3\xi+1}}
\]
\begin{enumerate}
\item Unlearning: $\R_\A$ is a weak $(\epsilon, \delta)$-unlearning algorithm for $\A$ with respect to $\publish$.
\item Accuracy: for any $\beta$, $(\A, \R_\A)$ is $(\alpha,\beta)$-accurate with respect to $\publish$ where
\[
\alpha  = \bigo \left(  \left( \frac{M^{\frac{1+\xi}{\xi}} L^2 D^{\frac{2+4\xi}{\xi}} d \log \left( 1/\delta \right)}{\epsilon^2 n^2 \I^{\frac{1}{\xi}}} \right)^{\frac{\xi}{3\xi+1}} \log^2 \left( d / \beta \right) \right) + \bigo \left( n^{-\frac{4 \xi}{3\xi+1}} \right) + \bigo \left( n^{-\frac{6 \xi}{3\xi+1}} \right)
\]
\end{enumerate}
\end{remark}
The proof of Theorem~\ref{thm:c-smooth-guarantees-weak} can be found in Appendix~\ref{sec:app-weak}.
}

\section{Perturbed Distributed Descent}
\label{sec:distributed}
%!TEX root = main-full.tex
Our next algorithm obtains additional running time improvements for sufficiently high dimensional data.  The basic idea is as follows: we randomly partition the dataset into $K$ parts, separately optimize to find a model that approximates the empirical risk minimizer on each part, and then take the average of each of the $K$ models. \longversion{Zhang et al \cite{ZDW12} analyze this algorithm and show that its \emph{out of sample guarantees} match the out of sample guarantees of non-distributed gradient descent, whenever $K \leq \sqrt{n}$.} For us, this algorithm has a key advantage: the element involved in an update will only appear in a small number of the partitions, and we only need to update the parameters corresponding to those partitions. Our algorithm will improve over basic gradient descent because those partitions are smaller in size than the entire dataset by a factor of $K$, and hence our run-time budget of $n \I$ gradient computations will allow us to perform more than $\I$ gradient descent operations per modified partition.  We provide deletion guarantees by using a publishing function that adds noise to the average of the $K$ parameters.

There are several difficulties that we must overcome. Primary among these is that the analysis of \cite{ZDW12} provides out of sample guarantees for a dataset that is drawn $i.i.d.$ from some fixed distribution. In our case (because our dataset results from an arbitrary and possibly adversarial sequence of additions and deletions), there is no distribution from which the dataset is drawn. To deal with this, our initial training algorithm does not directly partition the dataset, but instead draws a \emph{bootstrap} sample (i.e. a sample with replacement) from the empirical distribution defined by the dataset, so that the ``out of sample'' guarantees of \cite{ZDW12} correspond to empirical risk bounds in our case. Because the accuracy analysis depends on this distributional property, as updates come in, before we use gradient descent to update the models corresponding to the appropriate partitions, we must apply a form of reservoir sampling to guarantee that each partition continues to be distributed as a set of samples drawn $i.i.d.$ from the empirical distribution defined by the \emph{current} dataset (i.e. after the update). This is also crucial to our unlearning guarantee. Finally, the basic instantiation of this algorithm only gives guarantees on the \emph{expected} error of the learned model \cite{ZDW12}, and we want high probability guarantees. To achieve these, we run $C = \bigo(\log (1/\beta))$ copies of the algorithm in parallel, and at every round, only \emph{publish} a noisy version of the parameter achieving the lowest loss among all $C$ candidates. \longversion{We now go into more detail.} \shortversion{We show the unlearning and publishing algorithms here: the learning algorithm, which is similar, is in the supplement.}

\longversion{
To facilitate the technical development in this section, we introduce some notation:
\begin{definition}
\label{def:newdefs}
Fix any update round $i \ge 0$. In this section we use $\boldsymbol{\cS}_i = (\cS_{ij})_{j=1}^K$ for the partitioned dataset at round $i$. We use $\cS_i$ (unbold) to denote the union of partitions in $\boldcS_i$ and $\D_i$ for the unique data points in $\cS_i$ (i..e $\D_i$ removes the duplicates in $\cS_i$ which results from our sampling scheme). We use $\boldsymbol{\thetahat}_i = (\thetahat_{ij})_{j=1}^K$ for the learned parameters in each partition. $\thetatilde_i = \publish ( \boldthetahat_i )$ represents the published model of round $i$. In this section, the unlearning algorithm for update $i$ takes as input the partitioned dataset of previous round $\boldsymbol{\cS}_{i-1}$, an update $u_i$, and the learned models of previous round $\boldthetahat_{i-1}$, and outputs the updated models $\boldthetahat_i$ and the updated datasets $\boldcS_i$ for use in the next update.
\end{definition}
}

\longversion{
\begin{algorithm}[h]
	\begin{algorithmic}[1]
		\State Input: dataset $\D$
		\For{$l = 1,2, \ldots, C$}
			\State Draw $\cS \sim \mathcal{P}^B(\D)$. \Comment{Bootstrap $B$ data points.}
			\State Partition $\cS$ randomly into $K$ equally-sized datasets: $\boldcS_{0,l} = (\cS_{j})_{j=1}^K$.
			\For{$j = 1,2, \ldots, K$}
				\State Initialize $\theta'_0 \in \Theta$.
				\For {$t=1,2, \ldots T$}
					\State $\theta'_t = \proj_\Theta \left( \theta'_{t-1}-\eta_t \nabla f_{\cS_{j}} (\theta'_{t-1}) \right)$
				\EndFor
				\State $\thetahat_{j} = \theta'_T$
			\EndFor
			\State $\boldthetahat_{0,l} = (\thetahat_{j})_{j =1}^{K}$ \Comment{$l$'th set of models.}
		\EndFor
		\State Call $\publish (\boldthetahat_{0,l^*})$ where $l^* = \argmin_l  f_\D ( \text{avg} (\boldthetahat_{0,l}) )$. \Comment{Publish the best model.}
		\State Output: $\boldthetahat_0 = (\boldthetahat_{0,l})_{l=1}^{C}, \boldcS_0 = (\boldcS_{0,l})_{l=1}^C$ \Comment{For use in first update.}
\end{algorithmic}
\caption{$\A$: \textbf{Learning} for Perturbed Distributed Gradient Descent}
\label{alg:dist}
\end{algorithm}
}
\longversion{
Throughout we denote the distribution on datasets of size $B$ sampled \textit{with replacement} from $\D$ by $\P^B(\D)$. We need to maintain the condition that the marginal distribution of the sampled dataset $\cS_i$ at round $i$ is $\P^B (\D_i)$. To do this, at each update, we iteratively update each partition using a technique called reservoir sampling with replacement (that we need to extend to handle both additions and deletions). The algorithm $\mathcal{S}_{rep}^B$ is detailed below.
}

\begin{algorithm}[h]
\begin{algorithmic}
	\State Input: Subsample $\cS_{i-1}$, update $u_i = (z_i, \bullet_i)$.
		\State $\cS_i = \cS_{i-1}$.
		\If{$\bullet_i = \mathtt{'add'}$}
			\State Draw $N \sim \text{Binomial}(B, n_i^{-1})$. \Comment{$n_i$: size of $\D_i$.}
			\State Pick distinct indices $i_1, \ldots , i_N$ at random from $[B]$.
			\For{$k=1,2,\ldots,N$}
				\State Replace $z_{i_k}$ with $z_i$ in $\cS_i$.
			\EndFor
		\Else
			\For{$z_k \in \cS_i: z_k = z_i$ }
				\State Replace $z_k$ with $z \sim \mathcal{P}(\D_i)$ in $\cS_i$ \Comment{$\mathcal{P}(\D_i)$: empirical distribution of $\D_i$}.
			\EndFor
		\EndIf
	\State Output: $\cS_{i}$.
\end{algorithmic}
\caption{$\mathcal{S}_{rep}^B$: Reservoir Sampling with Replacement for $i$th update}
\label{alg:rswr}
\end{algorithm}

%\begin{algorithm}[h]
%\begin{algorithmic}
%	\State Input: $\D = \D_0, \U, q$
%	\For{$u_i \in \U$}
%		\State Update $\D_i = \D_{i-1} \circ u$
%		\If{$i = 0$}
%		\State Draw $S_0 \sim \mathcal{P}_m(\D_0)$
%		\Else
%			\If{$u_i = (z_i, \mathtt{'add'}$)}
%				\State Draw $N \sim \text{Bin}(q, \frac{1}{n+1})$ \ar{What is $n$? Should this be $n_i$? Why is the $+1$ there?}
%				\State Pick distinct indices $i_1, \ldots , i_N$ at random from $[q]$
%				\State For each $s_{i_k} \in S_{i-1}$, set $s_{i_k} = z_i$.
%				\State Update $S_i = \{s_l\}_{l=1}^{q}$
%			
%			\Else
%				\For{All $s_l \in S_{i-1}: s_l = z_i$ }
%					\State Draw a new $s_l \sim \mathcal{P}(\D_i)$
%				\EndFor
%				\State Update $S_i = \{s_l\}_{l=1}^{q}$
%			\EndIf
%		\EndIf
%	\EndFor
%\end{algorithmic}
%\caption{$\mathcal{S}^{rep}_m$: Reservoir Sampling with Replacement}
%\label{alg:rswr}
%\end{algorithm}
\longversion{
\begin{lemma}
\label{samp}
Fix any training dataset $\D$ and any non-adaptively chosen update sequence $\U$. Let $\cS_0 \sim \P^B (\D)$ (as in the learning algorithm) and for every $i \ge 1$, $\cS_i \sim \mathcal{S}_{rep}^B (\cS_{i-1}, u_i)$ (as in the unlearning algorithm). We have that for all $i \ge 0$:
$$
\cS_i \stackrel{d}{=} \mathcal{P}^{B}(\D_i).
$$
\end{lemma}
}

\longversion{
Lemma~\ref{samp2} shows that the reservoir sampling operation (Algorithm~\ref{alg:rswr}) modifies at most $s_i = \bigotilde (B/n)$ data points, and hence, at most $s_i$ partitions containing a modified data point. Thus we can divide our budget of $nT_i$ gradient computations at round $i$, into $(K nT_i)/(B s_i)$ gradient computations per modified partition.

\begin{lemma}
\label{samp2}
Fix any training dataset $\D$ and any update sequence $\U$, and suppose $B \ge n$. Let $s_i$ denote the number of data points modified by the update of round $i$, namely, $u_i$. In other words, $s_i = \vert \{ z_l : z_l \in \cS_i , z_l \notin \cS_{i-1}  \} \vert$. We have that for any update step $i$ and any $\delta' \le e^{-1}$, with probability at least $1-\delta'$,
\[
s_i \leq \frac{10B}{n} \log\left( 1 / \delta' \right)
\]
\end{lemma}
}

\begin{algorithm}[H]
	\begin{algorithmic}[1]
		\State Input: datasets $\boldcS_{i-1} = (\boldcS_{i-1,l})_{l=1}^C$, update $u_i$, models $\boldthetahat_{i-1} = (\boldthetahat_{i-1,l})_{l=1}^C$.
		\State Update $\D_i = \D_{i-1} \circ u_i$.
		\For{$l = 1,2, \ldots, C$}
			\State Draw $\boldcS_{i,l}\sim  \mathcal{S}_{rep}^B (\cS_{i-1,l}, u_i)$ \Comment{Reservoir update + similar partition.}
			\State Let $(\cS_{i,j})_{j=1}^K \equiv \boldcS_{i,l}$, $(\cS_{i-1,j})_{j=1}^K \equiv \boldcS_{i-1,l}$, $(\thetahat_{i-1,j})_{j=1}^K \equiv \boldthetahat_{i-1,l}$.
			\State Let $\mathtt{ind} = \{j: \cS_{i-1,j} \neq \cS_{i, j} \}$ \Comment{Modified partitions.}
			\For{$j=1,2, \ldots, K$}
				\If{$j \in \mathtt{ind}$}
					\State Initialize $\theta'_0 = \thetahat_{i-1,j}$
					\For {$t=1,2, \ldots, T = \frac{K n T_i}{B \vert \mathtt{ind} \vert}$}
						\State $\theta'_t = \proj_\Theta \left( \theta'_{t-1}-\eta_t \nabla f_{\cS_{i,j}} (\theta'_{t-1}) \right)$
					\EndFor
					\State $\thetahat_{i,j} = \theta'_T$
				\Else
					\State $\thetahat_{i, j} = \thetahat_{i-1, j}$
				\EndIf
			\EndFor
		\State $\boldthetahat_{i,l} = (\thetahat_{i,j})_{j =1}^{K}$ \Comment{$l$'th set of models.}
		\EndFor
		\State Call $\publish (\boldthetahat_{i,l^*})$ where $l^* = \argmin_l  f_{\D_i} ( \text{avg} (\boldthetahat_{i,l}) )$ \Comment{Publish the best model.}
		\State Output: $\boldthetahat_i = (\boldthetahat_{i,l})_{l=1}^{C}, \boldcS_i = (\boldcS_{i,l})_{l=1}^C$. \Comment{For use in next update.}
	\end{algorithmic}
\caption{$\R_\A$: $i$th \textbf{Unlearning} for Perturbed Distributed Gradient Descent}
\label{alg:splitupdate}
\end{algorithm}

\begin{algorithm}[H]
	\begin{algorithmic}[1]
		\State Input: $\boldthetahat = (\thetahat_{j})_{j=1}^{K}$
		\State Draw $Z \sim \mathcal{N} \left( 0, \sigma^2 \Imat_d \right)$
		\State Output: $\thetatilde = \text{avg} (\boldthetahat) + Z$ \Comment{$\text{avg} (\cdot)$ averages input models.}
	\end{algorithmic}
\caption{$\publish$: publishing function}
\label{alg:publish-2}
\end{algorithm}

%\begin{algorithm}[h]
%	\begin{algorithmic}[1]
%		\State Input: $\D_{i-1}, S_{i-1}, \U, u_i, P^{K, q}, \eta, T_i, (\theta_{i-1, j})_{j=1}^{K}$
%		\State Draw $S_i \sim  \mathcal{S}^{rep}_m(\D, \U)_i$
%		\State Let $ind = \{l \in [K] : S_{i-1,l} \neq S_{i, l} \}$ \Comment{partitions where a point was modified after the update}
%		\For{$l \in ind$}
%			\State Initialize $\theta_0 = \theta_{i-1,l}$
%			\For {$t=1,2, \ldots \frac{n T_i}{q |ind|}$}
%				\State Compute $g_t = \nabla f_{S_l}(\theta_{t-1})$				
%				\State Update $\theta_{t}= \theta_{t-1}-\eta_t g_t$
%			\EndFor
%			\State Save $\theta_{i,l} = \theta_{t}$
%		\EndFor
%		\For{$l \not \in ind$}
%			\State $\theta_{i, l} = \theta_{i-1, l}$
%		\EndFor
%		\State Output: $(\theta_{i, l})_{l = 1}^{K}$
%	\end{algorithmic}
%\caption{$\R_\A$: $i$th \textbf{Unlearning} for Perturbed Distributed Gradient Descent}
%\label{alg:splitupdate}
%\end{algorithm}

\longversion{
We now state the accuracy and strong unlearning bounds for perturbed distributed gradient descent. The convergence analysis on each partition is similar to the analysis in the proof of Theorem \ref{thm:sc-smooth-guarantees}, with the added complexity of handling the number of partitions updated at each round, and the number of duplicated points (that could possibly be removed) in each partition. In order to obtain accuracy bounds we need to leverage an accuracy bound for the averaged parameter in a distributed setting, which we quote below from \cite{ZDW12}. They remark that the required assumptions hold in most common settings, including in linear and logistic regression as long as the data distribution satisfies standard regularity conditions.

\begin{theorem}[Corollary 2 of \cite{ZDW12}]
\label{thm:zdw}
Let $\thetastar_{\text{avg}} = K^{-1} \sum_{j =1}^{K}\thetastar_{j}$, where $\thetastar_{j}$ are the empirical risk minimizers on partition $j$ of a dataset of size $B$ sampled $i.i.d.$ from some distribution $\cP$. Let $\thetastar = \argmin_{\theta \in \Theta} \mathbb{E}_{z \sim \cP} [f_z (\theta)]$. Then under the assumption that $f_z$ is $m$-strongly convex for all $z$, and satisfies the following smoothness conditions for all $\theta \in \Theta$:
\[
\mathbb{E}_{z \sim \cP} \left[ \left\Vert \nabla f_z (\theta) \right\Vert_2^8 \right] \leq L^{8}, \;\; \mathbb{E}_{z \sim \cP} \left[ \vertiii{ \nabla^2 f_z (\theta)- \nabla^2 \mathbb{E}_{z \sim \cP} \left[ f_z (\theta) \right] }_2^8 \right] \leq H^{8},
\]
and the Hessian matrix $\nabla^2 f_z (\cdot)$ is $G$-Lipschitz continuous for all $z$, then, for some constant $c$:
\[
\mathbb{E}\left[ \left\Vert \thetastar_{\text{avg}}-\thetastar \right\Vert_2^2 \right] \leq \frac{2L^2}{m^2 B} + \frac{cK^2L^2}{m^4 B^2} \left(H^2\log d + \frac{L^2G^2}{m^2} \right) + \bigo \left(\frac{K}{B^2} \right) + \bigo  \left(\frac{K^3}{B^3} \right)
\]
\end{theorem}
}

\shortversion{\begin{theorem}[Accuracy, Unlearning, and Computation Tradeoffs]
\label{thm:scgd-distributed}
Suppose the loss function $f_z$ is $m$-strongly convex, $L$-Lipschitz, $M$-smooth, and that its Hessian is $G$-Lipschitz and bounded by $H$. Define $\gamma \triangleq (M-m)/(M+m)$, and fix any $1 \le \xi \le 4/3$. Then there are parameters such that:
\begin{enumerate}
\item Unlearning: $\R_\A$ is an $(\epsilon, \delta)$-strong unlearning algorithm for $\A$ with respect to $\publish$.
\item Accuracy: for any $\beta$, $(\A, \R_\A)$ is $(\alpha,\beta)$-accurate with respect to $\publish$ where
\[
\alpha = \bigo \left( \frac{\gamma^{\I n^{\frac{4-3\xi}{2}}} d \log \left( 1/\delta \right) \log^2 \left(d/\delta \right) }{\epsilon^2n^2 }\right)  + \bigo \left( \frac{\log d}{n^{\xi}} \right) + \bigo \left( \frac{1}{n^{\frac{3\xi}{2}}} \right)
\]
\end{enumerate}
\end{theorem}
}

\longversion{
\begin{theorem}[Accuracy, Unlearning, and Computation Tradeoffs]
\label{thm:scgd-distributed}
Suppose for all $z \in \Z$, the loss function $f_z$ is $m$-strongly convex, $L$-Lipschitz, $M$-smooth, and that its Hessian is $G$-Lipschitz and bounded by $H$ (with respect to $\ell_2$-operator norm of matrices). Define $\gamma \triangleq (M-m)/(M+m)$ and $\eta \triangleq 2/(M+m)$. Fix any $1 \le \xi \le 4/3$, and let $B = n^\xi$ and $K = \sqrt{B}$. Let the learning algorithm $\A$ (Algorithm~\ref{alg:dist}) run with $\eta_t = \eta$  and $T$ iterations on every partition, and for any update $i \ge 1$, let the unlearning algorithm $\R_\A$ (Algorithm~\ref{alg:splitupdate}) run with $\eta_t = \eta$ and total $T_i $ iterations per copy (i.e. total $nT_i$ gradient computations per copy), where for any $\I$,
\[
T \ge \I n^{\frac{4-3\xi}{2}} + \frac{\log \left( Dm L^{-1} n^\xi \left( 1 + 10 \log\left( 2 / \delta \right) \right) \right)}{\log \left( 1 / \gamma \right)}
\]
\[
T_i = 10\log\left( 2i / \delta \right) \left( \I + \frac{1}{n^{\frac{4 - 3\xi}{2}}} \cdot \frac{\log \left( 1 + 10 i \log\left( 2i / \delta \right) \right)}{ \log \left( 1/ \gamma \right)} \right)
\]
Let the unlearning parameters $\epsilon$ and $\delta$ be such that $\epsilon = \bigo \left(\log \left( 1/\delta \right) \right)$ and $\delta = \bigo (B^{-1})$, and let
\[
\sigma = \frac{ 4 \sqrt{2} L \gamma^{\I n^{\frac{4-3\xi}{2}}}}{ mn \left( 1 - \gamma^{\I n^{\frac{4-3\xi}{2}}} \right) \left(\sqrt{\log \left( 2/\delta \right) + \epsilon} - \sqrt{\log \left( 2/\delta \right)} \right)}
\]
in $\publish$ (Algorithm~\ref{alg:publish-2}). We have that
\begin{enumerate}
\item Unlearning: $\R_\A$ is a strong $(\epsilon, \delta)$-unlearning algorithm for $\A$ with respect to $\publish$.
\item Accuracy: for any $\beta$, letting $C = \log \left( 2/\beta \right) / \log 2$, we get that $(\A, \R_\A)$ is $(\alpha,\beta)$-accurate with respect to $\publish$ where
\[
\alpha = \bigo \left( \frac{M L^2 \gamma^{2\I n^{\frac{4-3\xi}{2}}} d \log \left( 1/\delta \right) \log^2 \left(d/\beta \right) }{m^2 (1-\gamma)^2 \epsilon^2n^2 }\right)  + \bigo \left( \frac{\log d}{n^{\xi}} \right) + \bigo \left( \frac{1}{n^{\frac{3\xi}{2}}} \right)
\]
\end{enumerate}
\end{theorem}
}

\longversion{
\begin{remark}
\label{distwin}
For any $1 \le \xi \le 4/3$:
\[
\alpha = \bigotilde \left( \frac{\gamma^{\I n^{\frac{4-3\xi}{2}}} d }{\epsilon^2n^2 } + \frac{1}{n^{\xi}} \right)
\]
This improves over the bound of Theorem~\ref{thm:sc-smooth-guarantees} (Basic Perturbed GD), whenever
\[
d = \tilde{\Omega} \left( \frac{\epsilon^2 n^{2-\xi}}{\gamma^\I - \gamma^{\I n^{\frac{4-3\xi}{2}}} } \right)
\]
\end{remark}
The proof of Theorem~\ref{thm:scgd-distributed} can be found in Appendix~\ref{sec:app-distributed}.
}
\shortversion{
\begin{remark}
\label{distwin}
For any $1 \le \xi \le 4/3$, this improves over the bound of Theorem~\ref{thm:sc-smooth-guarantees} whenever
\[
d = \tilde{\Omega} \left( \frac{\epsilon^2 n^{2-\xi}}{\gamma^\I - \gamma^{\I n^{\frac{4-3\xi}{2}}} } \right)
\]
\end{remark}
}

\bibliographystyle{alpha}
\bibliography{refs}

\newcommand{\etalchar}[1]{$^{#1}$}
\begin{thebibliography}{GGHvdM19}

\bibitem[ABD17]{washu}
Aleksandr Aravkin, James Burke, and Dmitriy Drusvyatskiy.
\newblock Convex analysis and nonsmooth optimization, 2017.

\bibitem[AJM19]{pan2}
Kareem Amin, Matthew Joseph, and Jieming Mao.
\newblock Pan-private uniformity testing.
\newblock {\em arXiv preprint arXiv:1911.01452}, 2019.

\bibitem[BCCC{\etalchar{+}}19]{unlearning}
Lucas Bourtoule, Varun Chandrasekaran, Christopher Choquette-Choo, Hengrui Jia,
  Adelin Travers, Baiwu Zhang, David Lie, and Nicolas Papernot.
\newblock Machine unlearning.
\newblock {\em arXiv preprint arXiv:1912.03817}, 2019.

\bibitem[BS16]{BS16}
Mark Bun and Thomas Steinke.
\newblock Concentrated differential privacy: Simplifications, extensions, and
  lower bounds.
\newblock In {\em Theory of Cryptography Conference}, pages 635--658. Springer,
  2016.

\bibitem[Che19]{pton}
Yuxin Chen.
\newblock Ele 522 lecture notes: Gradient methods for unconstrained problems,
  2019.

\bibitem[CY15]{CY15}
Yinzhi Cao and Junfeng Yang.
\newblock Towards making systems forget with machine unlearning.
\newblock In {\em 2015 IEEE Symposium on Security and Privacy}, pages 463--480.
  IEEE, 2015.

\bibitem[CZW{\etalchar{+}}20]{DifferencingPrivacy}
Min Chen, Zhikun Zhang, Tianhao Wang, Michael Backes, Mathias Humbert, and Yang
  Zhang.
\newblock When machine unlearning jeopardizes privacy.
\newblock {\em arXiv preprint arXiv:2005.02205}, 2020.

\bibitem[DMNS06]{DMNS06}
Cynthia Dwork, Frank McSherry, Kobbi Nissim, and Adam Smith.
\newblock Calibrating noise to sensitivity in private data analysis.
\newblock In {\em Theory of cryptography conference}, pages 265--284. Springer,
  2006.

\bibitem[DNPR10]{pan1}
Cynthia Dwork, Moni Naor, Toniann Pitassi, and Guy~N Rothblum.
\newblock Differential privacy under continual observation.
\newblock In {\em Proceedings of the forty-second ACM symposium on Theory of
  computing}, pages 715--724, 2010.

\bibitem[DR14]{DR14}
Cynthia Dwork and Aaron Roth.
\newblock The algorithmic foundations of differential privacy.
\newblock {\em Foundations and Trends{\textregistered} in Theoretical Computer
  Science}, 9(3--4):211--407, 2014.

\bibitem[GGHvdM19]{hessian}
Chuan Guo, Tom Goldstein, Awni Hannun, and Laurens van~der Maaten.
\newblock Certified data removal from machine learning models.
\newblock {\em arXiv preprint arXiv:1911.03030}, 2019.

\bibitem[GGVZ19]{forgetu}
Antonio Ginart, Melody~Y. Guan, Gregory Valiant, and James Zou.
\newblock Making {AI} forget you: Data deletion in machine learning.
\newblock {\em CoRR}, abs/1907.05012, 2019.

\bibitem[ISCZ20]{izzo2020approximate}
Zachary Izzo, Mary~Anne Smart, Kamalika Chaudhuri, and James Zou.
\newblock Approximate data deletion from machine learning models: Algorithms
  and evaluations.
\newblock {\em arXiv preprint arXiv:2002.10077}, 2020.

\bibitem[MMRyA16]{McMahan}
H.~Brendan McMahan, Eider Moore, Daniel Ramage, and Blaise~Ag{\"{u}}era
  y~Arcas.
\newblock Federated learning of deep networks using model averaging.
\newblock {\em CoRR}, abs/1602.05629, 2016.

\bibitem[NRS07]{nissim2007smooth}
Kobbi Nissim, Sofya Raskhodnikova, and Adam Smith.
\newblock Smooth sensitivity and sampling in private data analysis.
\newblock In {\em Proceedings of the thirty-ninth annual ACM symposium on
  Theory of computing}, pages 75--84, 2007.

\bibitem[PAE{\etalchar{+}}16]{pate}
Nicolas Papernot, Martin Abadi, Ulfar Erlingsson, Ian Goodfellow, and Kunal
  Talwar.
\newblock Semi-supervised knowledge transfer for deep learning from private
  training data, 2016.

\bibitem[SSSS17]{attack}
Reza Shokri, Marco Stronati, Congzheng Song, and Vitaly Shmatikov.
\newblock Membership inference attacks against machine learning models.
\newblock In {\em 2017 IEEE Symposium on Security and Privacy (SP)}, pages
  3--18. IEEE, 2017.

\bibitem[Vit85]{rswr}
Jeffrey~S. Vitter.
\newblock Random sampling with a reservoir.
\newblock {\em ACM Trans. Math. Softw.}, 11(1):37?57, March 1985.

\bibitem[ZDW12]{ZDW12}
Yuchen Zhang, John~C. Duchi, and Martin Wainwright.
\newblock Comunication-efficient algorithms for statistical optimization, 2012.

\end{thebibliography}

\longversion{
\appendix

\section*{Appendix}
%!TEX root = main-full.tex
\section{Probabilistic Tools}

\begin{lemma}\label{lem:high-prob-zcdp}
Suppose $X,Y$ are random variables over the same domain $\Omega$, and let $Z$ be any random variable. If with probability at least $1-\delta$ over $Z$, we have $X \vert Z \overset{\epsilon,\delta}{\approx} Y \vert Z$, then $X \overset{\epsilon,2\delta}{\approx} Y$.
\end{lemma}

\begin{proof}
Define for any $z$, the following (good) event:
$$
G(z) = \left\{ z: X \vert \left(Z=z\right) \overset{\epsilon,\delta}{\approx} Y \vert \left(Z=z\right) \right\}
$$
and note that $\Pr_{z \sim Z} \left[ z \notin G(Z) \right] \le \delta$. We have that for any $S \subseteq \Omega$,
\begin{align*}
\prob{X \in S} &= \mathbb{E}_{z \sim Z} \left[ \prob{X \in S \vert Z =z}\right] \\
&=  \mathbb{E}_{z \sim Z} \left[ \prob{X \in S \vert Z =z} \mathds{1} \left(z \in G(z) \right) + \prob{X \in S \vert Z =z} \mathds{1} \left(z \notin G(z) \right) \right] \\
&\le \mathbb{E}_{z \sim Z} \left[ e^\epsilon \prob{Y \in S \vert Z =z} + \delta \right] + \Pr_{z \sim Z} \left[z \notin G(z) \right] \\
&\le e^\epsilon \mathbb{E}_{z \sim Z} \left[ \prob{Y \in S \vert Z =z} \right] + 2 \delta \\
&= e^\epsilon \prob{Y \in S} + 2 \delta
\end{align*}
where $\mathds{1} (A)$ is the indicator function of event $A$, for any $A$. This completes the proof because we can similarly show,
$$
\prob{Y \in S} \le e^\epsilon \prob{X \in S} + 2 \delta
$$
\end{proof}

\begin{lemma}[Gaussian Tail Bound]\label{lem:normal-tail}
Let $Z \sim \N (0,\sigma^2 \Imat_d)$. We have that for any $\beta > 0$,
$$
\prob{\left\lVert Z \right\lVert_2 \ge \sigma \sqrt{2d} \log(2d/\beta)} \le \beta
$$
\end{lemma}

%\begin{proof}[Proof]
%Note that for any $t \ge 0$, and for a one dimensional Gaussian $Z_0 \sim \N (0,\sigma^2)$,
%$$
%\prob{ \left\vert Z_0 \right\vert \ge t} \le 2 \exp \left( \frac{-t^2}{2\sigma^2} \right)
%$$
%Therefore, for any $t \ge 0$,
%\begin{align*}
%\prob{\left\lVert Z \right\lVert_2 \ge t} &\le \sum_{i=1}^d  \prob{\left\vert Z_i \right\vert \ge \frac{t}{\sqrt{d}}} \le 2d\exp \left( \frac{-t^2}{2d\sigma^2} \right)
%\end{align*}
%\end{proof}

\begin{lemma}[Gaussian Mechanism \cite{BS16}]\label{lem:zcdp}
Let $X \sim \N ( \mu ,\sigma^2 \Imat_d)$ and $Y \sim \N ( \mu' ,\sigma^2 \Imat_d)$. Suppose $\left\Vert \mu - \mu' \right\Vert_2 \le \Delta$. We have that for any $\delta > 0$, $X \overset{\epsilon,\delta}{\approx} Y$, where
$$
\epsilon = \frac{\Delta^2}{2 \sigma^2} + \frac{\Delta}{\sigma}\sqrt{2 \log \left( 1 / \delta \right)}
$$
\end{lemma}

\begin{lemma}\label{lem:exp-to-highprob}
Let $X \geq 0$ be any random variable drawn from a distribution $\P$, with finite expectation $\mu = \mathbb{E}_{X \sim \P}[X]$. 
Let $X_1, \ldots X_N \overset{iid}{\sim} \P$. Then if $X_{min} \triangleq \min_jX_j$, for $N \geq \frac{\log \left( 1/ \delta \right)}{\log 2}$, with probability at least $1-\delta$: $X_{\min} < 2\mu$.
\end{lemma}
\begin{proof}
By Markov's inequality, for any $X_j, \prob{X_j \geq 2\mu} \leq \frac{1}{2}$. Hence, 
\[
\prob{X_{min} \geq 2\mu} = \prod_{j=1}^{N} \prob{X_j \geq 2\mu} \leq \left( \frac{1}{2} \right)^N  \leq \left( \frac{1}{2} \right)^\frac{\log \left( 1/ \delta \right)}{\log 2} = \delta
\]
as desired. 
\end{proof}

\begin{lemma}[Chernoff Bound]\label{lem:chernoff}
\label{chernoff}
Let $X \sim \text{Binomial} \, (m, p)$, and let $\mu = mp$. Then for any $\delta' \ge 0$,
$$\Pr \left[X \geq (1+\delta') \mu \right] \leq e^{-\frac{\mu{\delta'^2}}{2+\delta'}}$$
\end{lemma}

\section{Proof of Sensitivity Lemma \ref{lem:par}}\label{app:par-proof}
To prove Lemma \ref{lem:par}, we will need the following claim.
\begin{claim}\label{clm:claim}
Suppose $h: \Theta \to \Real$ is $m$-strongly convex and let $\thetastar = \argmin_{\theta \in \Theta} h(\theta)$. We have that for any $\theta \in \Theta$,
$
h(\theta) \ge h(\thetastar) + \frac{m}{2} \left\lVert \theta - \thetastar \right\lVert_2^2
$.
\end{claim}
\begin{proof}
First, recall the definition of $m$-strong convexity: for any $\theta_1,\theta_2 \in \Theta$, and any $t \in (0,1)$,
$$
h ( t \theta_1 + (1-t) \theta_2 ) \le t h(\theta_1) + (1-t) h(\theta_2) - \frac{m}{2} t (1-t) \left\lVert \theta_1 - \theta_2 \right\lVert_2^2
$$
Now fix some $\theta \in \Theta$. We have that for any $t \in (0,1)$,
\begin{align*}
h(\thetastar) &\le h ( t\theta+(1-t) \thetastar ) \le t h(\theta) + (1-t) h(\thetastar) - \frac{m}{2} t (1-t) \left\lVert \theta - \thetastar \right\lVert_2^2
\end{align*}
where the first inequality follows because $\thetastar$ is the minimizer of $h$, and the second is due to $m$-strong convexity of $h$. Rearranging the above inequality and dividing both sides by $t$, we get that for \emph{any} $t \in (0,1)$,
$$
h(\theta) \ge h(\thetastar) + \frac{m}{2} (1-t) \left\lVert \theta - \thetastar \right\lVert_2^2
$$
We therefore have that
\begin{align*}
h(\theta) &\ge h(\thetastar) + \frac{m}{2} \sup_{t \in (0,1)}(1-t) \left\lVert \theta - \thetastar \right\lVert_2^2 = h(\thetastar) + \frac{m}{2} \left\lVert \theta - \thetastar \right\lVert_2^2
\end{align*}
\end{proof}

\begin{proof}[Proof of Lemma \ref{lem:par}]
Fix $n$, a data set $\D= \{z_i\}_{i =1}^n$, and an update $u = (z, \bullet)$, and let $\D' = \D \circ u$. Assume $\bullet = \mathtt{'delete'}$. If $z \notin \D$, then the claim immediately follows; so suppose $z \in \D$. We have that
\begin{align}\label{eq:upper}
\begin{split}
f_\D \left( \thetastar_{\D'} \right) &= \frac{n-1}{n} f_{\D'} \left( \thetastar_{\D'} \right) + \frac{1}{n} f_{z} \left(\thetastar_{\D'} \right) \\
&\le \frac{n-1}{n} f_{\D'} \left( \thetastar_{\D} \right) + \frac{1}{n} f_{z} \left(\thetastar_{\D'} \right) \\
&= f_\D \left( \thetastar_{\D} \right) +\frac{1}{n} f_{z} \left(\thetastar_{\D'} \right) - \frac{1}{n} f_{z} \left(\thetastar_{\D} \right) \\
&\le f_\D \left( \thetastar_{\D} \right) + \frac{L}{n} \left\Vert \thetastar_{\D'}- \thetastar_\D  \right\Vert_2
\end{split}
\end{align}
where the first inequality follows by optimality of $\thetastar_{\D'}$ for $\D'$, and the second follows by $L$-Lipschitzness of $f_{z}$. Note that since $f_\D$ is $m$-strongly convex, Claim \ref{clm:claim} implies
\begin{align}\label{eq:lower}
\begin{split}
f_\D \left(\thetastar_{\D'} \right) &\geq f_\D \left(\thetastar_\D \right) + \frac{m}{2} \left\lVert \thetastar_{\D'}- \thetastar_\D \right\lVert_2^2
\end{split}
\end{align}
Combining Equations (\ref{eq:upper}) and (\ref{eq:lower}) completes the proof for the case when $\bullet = \mathtt{'delete'}$. Note that when $\bullet = \mathtt{'add'}$, one can take $u' \triangleq (z,\mathtt{'delete'})$, and use the bound for deletion to conclude that
$$
\left\Vert \thetastar_{\D}-\thetastar_{\D \circ u} \right\Vert_2 = \left\Vert \thetastar_{\D \circ u} - \thetastar_{(\D \circ u) \circ u'}\right\Vert_2 \leq \frac{2L}{mn}
$$
\end{proof}

\section{Proofs of Lemmas in Section~\ref{sec:distributed}}

\begin{proof}[Proof of Lemma \ref{samp}]
We prove the claim by induction on $i$. For $i = 0$, $\cS_0$ is explicitly drawn from $\mathcal{P}^B (\D_0)$ and so the claim holds. Now assume the claim holds for $i-1$. In the case of addition, where $u_i = (z_i, \mathtt{'add'})$  this is exactly what is known as ``Reservoir Sampling with Replacement" and we refer the reader to \cite{rswr}. So we need only establish the claim for deletion updates.  Let us perform an update $u_i = (z_i, \mathtt{'delete'})$. We show that after conditioning on $u_i$, after the deletion update, each element of $\cS_i$ is independent and has marginal distribution $\P(\D_{i-1} \circ u_i) = \P(\D_i)$, which will establish the claim. Conditioning on $u_i$, let $h_{u_i}: \Z \to \Z$ be the randomized function:

\[ h_{u_i}(z) =  \begin{cases}
      z & z \neq z_i \\
      z' \sim \P(\D_i) & z = z_i
   \end{cases}
\]

Then for any data point $z_l \in \cS_{i-1}$, the corresponding element in $\cS_i$ is $h_{u_i} (z_l)$. Since by assumption the $\{z_l\} = \cS_{i-1}$ are independent, since $h_{u_i}$ is a fixed randomized function conditioned on $u_i$, the $\{h_{u_i} (z_l)\} = \cS_i$ are conditionally independent given $u_i$. It remains to show that the marginal distribution of any $z_l' = h_{u_i} (z_l)$ is $\P(\D_{i-1} \circ u_i) \equiv \P(\D_i)$. If $z_l = z_i$, then $z_l' \sim \P(\D_{i})$ by design. If $z_l \neq z_i$, then $z_l' = z_l$, and the distribution of $z_l'$ is $z_l | z_l \neq z_i, u_i$. Since $\U$ is a non-adaptive sequence of updates, $z_l | z_l \neq z_i, u_i \sim z_l | z_l \neq z_i$. Then by inductive assumption $z_l \sim \P(\D_{i-1})$, and so the distribution of $z_l | z_l \neq z_i$ for $z_i \in \D_{i-1}$ is uniform over $\D_{i-1} \setminus \{z_i \} = \D_{i-1} \circ u_i = \D_i$, which is exactly $\P(\D_i)$, as desired. This establishes the induction.
\end{proof}

\begin{proof}[Proof of Lemma~\ref{samp2}]
At any round $i$ of update, by Lemma~\ref{samp}, we know $\cS_i \sim \P^B (\D_i)$. By Assumption~\ref{ass:ass1}, $n_i \geq n/2$ where $n_i$ is the size of dataset $\D_i$. Hence for any data point $z$, the number of copies of $z$ subsampled in $\cS_i$ is distributed as $\text{Binomial}(B, p)$, where $p \leq 2/n$. Let $\mu = (2B)/n$ and note that $\mu \ge 1$. Now by a Chernoff bound (see Lemma~\ref{lem:chernoff}) for a Bernoulli random variable, we get that for any $i$, the number of repeated points of any one type in $\cS_i$ (including the ones subject to update) satisfies, with probability $1-\delta'$: 
 \begin{align*} \label{chern1}
 s_i &\leq \mu + \sqrt{\log^2 \left(1/\delta' \right) + 8 \mu \log \left(1/\delta' \right)} \\
 &= \mu \left( 1 + \sqrt{ \frac{\log^2 \left(1/\delta' \right)}{\mu^2} + \frac{8 \log \left(1/\delta' \right)}{\mu} } \right) \\
 &\le \mu \left( 1 + \sqrt{ \log^2 \left(1/\delta' \right) + 8 \log \left(1/\delta' \right) } \right) \\
 &\le 5 \mu \log \left(1/\delta' \right)
 \end{align*}
 as desired. Note the last inequality follows because $\log \left(1/\delta' \right) \ge 1$ by assumption.
 \end{proof}

\section{Proof of Theorem~\ref{thm:sc-smooth-guarantees-2}}\label{sec:app-perfect}
%!TEX root = main-full.tex
\begin{proof}[Proof of Theorem \ref{thm:sc-smooth-guarantees-2}]
We first prove the unlearning guarantee. Fix a training dataset $\D$ of size $n$ and an update sequence $\U = (u_i)_i$. Similar to the proof of Theorem~\ref{thm:sc-smooth-guarantees}, we first recall a few notations from Definition~\ref{def:notation}: $\{ \D_i \}_{i\ge0}$ for the sequence of updated datasets according to the update sequence $\U$, $\{ \thetahat_i \}_{i\ge0}$ for the sequence of secret non-noisy parameters, and $\{ \thetatilde_i \}_{i\ge0}$ for the sequence of published noisy parameters. Let $Z_i$ denote the Gaussian noise added by $\publish$ on round $i$ of update, and recall that $\thetatilde_i = \thetahat_i + Z_i$. We use $n_i$ ($\ge n/2$) to denote the size of $\D_i$. Let $\thetastar_i \triangleq \argmin_{\theta} f_{\D_i} (\theta)$ denote the optimizer of $f_{\D_i}$.

We have that for any $i \ge 0$, $\publish (\A \left(\D_i \right) ) \sim \N \left(\mu_i, \sigma^2 \Imat_d \right)$, where it follows by the convergence guarantee of Theorem \ref{thm:sc-sm-gd} that,
\begin{equation}\label{eq:mu-2}
%\begin{split}
\left\Vert \mu_i - \thetastar_{i} \right\Vert_2 \le \gamma^{T} \left\Vert \theta'_0 - \thetastar_{i} \right\Vert_2
= \frac{2L \gamma^\I \left\Vert \theta'_0 - \thetastar_{i} \right\Vert_2}{Dmn_i}
\le \frac{4L}{mn} \cdot \gamma^\I
%\end{split}
\end{equation}
We also have that for any update step $i \ge 1$, conditioned on the noise of previous rounds $\{Z_0, \ldots, Z_{i-1}\}$, $\publish \left( \R_\A \left( \D_{i-1}, u_i,  \theta_{i} \right) \right) \sim  \N \left(\mu'_i, \sigma^2 \Imat_d \right)$, where for any $\beta' > 0$,
\begin{equation}\label{eq:mu'-2}
\Pr_{Z_0, \ldots, Z_{i-1}} \left[ \left\Vert \mu'_i - \thetastar_{{i}} \right\Vert_2 \ge \frac{\gamma^{T_i}}{1 - \gamma^\I} \left( \frac{4L}{mn} + \sigma \sqrt{2d} \log \left( 2d/ \beta' \right) \right) \right] \le i\beta'
\end{equation}
We use induction on $i$ to prove this claim. Fix any $\beta'$. Let's focus on the base case $i=1$. We have that
\begin{align*}
 \left\Vert \mu'_1 - \thetastar_{1} \right\Vert_2 &\le \ \gamma^{T_1} \left\Vert \thetatilde_0 - \thetastar_{1} \right\Vert_2 \\
 &\le \gamma^{T_1} \left( \left\Vert Z_0 \right\Vert_2 + \left\Vert \thetahat_0 - \thetastar_{0} \right\Vert_2 + \left\Vert \thetastar_{0} - \thetastar_{1} \right\Vert_2 \right) \\
 &\le \gamma^{T_1} \left(  \frac{4L}{mn} \gamma^\I + \frac{4L}{mn} + \sigma \sqrt{2d} \log \left( 2d/ \beta' \right) \right) \\
 &\le  \gamma^{T_1} \left( \frac{\gamma^\I}{1 - \gamma^\I} \left( \frac{4L}{mn} + \sigma \sqrt{2d} \log \left( 2d/ \beta' \right) \right)  + \frac{4L}{mn} + \sigma \sqrt{2d} \log \left( 2d/ \beta' \right) \right) \\
 &=\frac{\gamma^{T_1}}{1 - \gamma^\I} \left( \frac{4L}{mn} + \sigma \sqrt{2d} \log \left( 2d/ \beta' \right) \right)
\end{align*}
The first inequality follows from Theorem \ref{thm:sc-sm-gd} and the fact that when running Algorithm~\ref{alg:del_scgd} for the first update $i=1$, the initial point of the algorithm $\theta'_0 = \theta_1 \equiv \thetatilde_0$. The second inequality is a simple triangle inequality, and the third holds with probability at least $1-\beta'$ and follows from Equation~(\ref{eq:mu-2}) (noting that $\thetahat_0 \equiv \mu_0$), the sensitivity Lemma~\ref{lem:par}, and a Gaussian tail bound for $Z_0$ (Lemma~\ref{lem:normal-tail}). Let's move on to the induction step of the argument. Suppose Equation~(\ref{eq:mu'-2}) holds for some $i\ge1$. We will show that it holds for $(i+1)$ as well. We have that
\begin{align*}
 \left\Vert \mu'_{i+1} - \thetastar_{{i+1}} \right\Vert_2 &\le \ \gamma^{T_{i+1}} \left\Vert \thetatilde_i - \thetastar_{{i+1}} \right\Vert_2 \\
 &\le \gamma^{T_{i+1}} \left( \left\Vert Z_i \right\Vert_2 + \left\Vert \thetahat_i - \thetastar_{i} \right\Vert_2 + \left\Vert \thetastar_{i} - \thetastar_{{i+1}} \right\Vert_2 \right) \\
 &\le \gamma^{T_{i+1}} \left( \frac{\gamma^{\I}}{1 - \gamma^\I} \left( \frac{4L}{mn} + \sigma \sqrt{2d} \log \left( 2d/ \beta' \right) \right)  + \frac{4L}{mn} + \sigma \sqrt{2d} \log \left( 2d/ \beta' \right) \right) \\
 &=\frac{\gamma^{T_{i+1}}}{1 - \gamma^\I} \left( \frac{4L}{mn} + \sigma \sqrt{2d} \log \left( 2d/ \beta' \right) \right)
\end{align*}
The first inequality follows from Theorem \ref{thm:sc-sm-gd} and the fact that when running Algorithm~\ref{alg:del_scgd} for the $(i+1)$th update, the initial point of the algorithm $\theta'_0 = \theta_{i+1} \equiv \thetatilde_{i}$. The second inequality is a simple triangle inequality, and the third holds with probability at least $1-(i+1)\beta'$ and follows from the induction assumption for $i$ (note $\thetahat_i \equiv \mu'_i$ and $T_i \ge \I$), the sensitivity Lemma~\ref{lem:par} (note $n_i \ge n/2$), and a Gaussian tail bound for $Z_i$ (Lemma~\ref{lem:normal-tail}). Now with the choice of $\beta' = \delta / (2i)$, Equation (\ref{eq:mu'-2}) implies with probability at least $1 - \delta/2$ over the Gaussian noise draws $\{Z_0, \ldots, Z_{i-1}\}$,
\begin{equation}\label{eq:mu'-3}
\left\Vert \mu'_i - \thetastar_{{i}} \right\Vert_2 \le \frac{\gamma^{\I}}{1 - \gamma^\I} \left( \frac{4L}{mn} + \sigma \sqrt{2d} \right)
\end{equation}
because $\gamma^{T_i} \le \left( \log \left( 4 d i / \delta \right) \right)^{-1}  \gamma^\I$. We therefore have shown that for any $i \geq 1$, conditioned on $\{Z_0, \ldots, Z_{i-1}\}$
\[
\publish \left(\A \left(\D_i \right) \right) \sim \N \left(\mu_i, \sigma^2 \Imat_d \right), \quad \publish \left( \R_\A \left( \D_{i-1}, u_i,  \theta_{i} \right) \right) \sim  \N \left(\mu'_i, \sigma^2 \Imat_d \right)
\]
where Equations~(\ref{eq:mu-2}) and (\ref{eq:mu'-3}) imply, with probability $1-\delta/2$ over $\{Z_0, \ldots, Z_{i-1}\}$,
\[
\left\Vert \mu_i - \mu'_i  \right\Vert_2 \le \frac{\gamma^{\I}}{1 - \gamma^\I} \left( \frac{4L}{mn} + \sigma \sqrt{2d} \right) + \frac{4L}{mn} \cdot \gamma^\I \le \frac{2\gamma^{\I}}{1 - \gamma^\I} \left( \frac{4L}{mn} + \sigma \sqrt{2d} \right) \triangleq \Delta
\]
It then follows from Lemma~\ref{lem:zcdp}, as well as the choice of $\sigma$ and the assumption on $\I$ in the theorem statement, that for any $i \geq 1$, with probability $1-\delta/2$ over $\{Z_0, \ldots, Z_{i-1}\}$,
\[
\publish \left(\A \left(\D_i \right) \right) \overset{\epsilon,\delta/2}{\approx} \publish \left( \R_\A \left( \D_{i-1}, u_i,  \theta_{i} \right) \right)
\]
Now we can apply Lemma~\ref{lem:high-prob-zcdp} to conclude that for any $i \geq 1$,
\[
\publish \left(\A \left(\D_i \right) \right) \overset{\epsilon,\delta}{\approx} \publish \left( \R_\A \left( \D_{i-1}, u_i,  \theta_{i} \right) \right)
\]
And this shows $\R_\A$ is an $(\epsilon, \delta)$-unlearning algorithm for $\A$, as desired.

Now let's prove the accuracy statement of the theorem. We will make use of Equations (\ref{eq:mu-2}) and (\ref{eq:mu'-3}) and a Gaussian tail bound (see Lemma~\ref{lem:normal-tail}). Recall that for any $i \ge 0$, the published output $\thetatilde_i = \thetahat_i + Z$, and that $\thetahat_0 \equiv \mu_0$ and $\thetahat_i \equiv \mu'_i$ for $i \ge 1$. We therefore have that, for any $\beta$, and for any update step $i\ge 0$,
\[
\Pr_{Z_0, \ldots, Z_{i}} \left[ \left\Vert \thetatilde_i -\thetastar_{{i}} \right\Vert_2  \ge   \frac{\gamma^{\I}}{1 - \gamma^\I} \left( \frac{4L}{mn} + \sigma \sqrt{2d} \right) + \sigma \sqrt{2d} \log \left(2d/\beta \right) \right] \le \beta + \frac{\delta}{2}
\]
The choice of $\sigma$ in the theorem and the fact that for $\epsilon = \bigo \left(\log \left( 1/\delta \right) \right)$, we have $\sqrt{\log \left( 1/\delta \right) + \epsilon} - \sqrt{\log \left( 1/\delta \right)} = \Omega ( \epsilon/ \sqrt{\log \left( 1/\delta \right)} )$, imply for any update step $i \ge 0$, with probability at least $1-\beta - \delta/2$,
\begin{equation}\label{eq:distance-smooth-2}
\left\Vert \thetatilde_i -\thetastar_{{i}} \right\Vert_2 = \bigo \left( \frac{L \gamma^{\I} \sqrt{d \log \left( 1/\delta \right)} \log \left(d/\beta \right) }{\left( 1 - \gamma^\I \right) \epsilon m n } \right)
\end{equation}
Finally, since $f_z$ is $M$-smooth for all $z$, we get that for any update step $i \ge 0$, with probability at least $1-\beta - \delta/2$,
\[
f_{\D_i} ( \thetatilde_i ) - f_{\D_i} ( \thetastar_{i} ) \le \frac{M}{2} \left\Vert \thetatilde_i -\thetastar_{{i}} \right\Vert_2^2 = \bigo \left( \frac{M L^2 \gamma^{2\I} d \log \left( 1/\delta \right) \log^2 \left(d/\beta \right) }{\left( 1 - \gamma^\I \right)^2 m^2  \epsilon^2n^2 } \right)
\]
\end{proof}

\section{Proof of Theorem~\ref{thm:c-smooth-guarantees-weak}}\label{sec:app-weak}
%!TEX root = main-full.tex
\begin{proof}[Proof of Theorem~\ref{thm:c-smooth-guarantees-weak}]
We first prove the unlearning guarantee. Fix a training dataset $\D$ of size $n$ and an update sequence $\U = (u_i)_i$. Recall from Definition~\ref{def:notation} the notation we use: $\{ \D_i \}_{i\ge0}$ for the sequence of updated datasets according to the update sequence $\U$, $\{ \thetahat_i \}_{i\ge0}$ for the sequence of secret non-noisy parameters, and $\{ \thetatilde_i \}_{i\ge0}$ for the sequence of published noisy parameters. We also use $n_i$ to denote the size of $\D_i$. Note that $n_0 = n$ and that by Assumption~\ref{ass:ass1}, $n_i \ge n/2$ for all $i$. Let $\thetastar_i \in \argmin_{\theta} f_{\D_i} (\theta)$ denote an optimizer of $f_{\D_i}$. Let $\thetastarr_i = \argmin_{\theta \in \Theta} g_{\D_i} (\theta)$ denote the optimizer of the regularized loss $g_{\D_{i}}$.

\begin{fact}\label{fact:dunno}
Note that for any positive integer $T'$,
\begin{equation}\label{eq:approx-weak}
\gamma^{T'} = \left( \frac{1}{1 + 2 \left(m/M \right)} \right)^{T'} \le \frac{1}{1+2 \left(m/M \right)T'} \le \sqrt{\frac{M}{m T'}}
\end{equation}
where the last inequality follows because for all $x \ge 0$, $1+x \ge 2 \sqrt{x}$.
\end{fact}

\begin{fact}[Generalizing Fact~\ref{fact:dunno}]\label{fact:dunno2}
In general, for any constant $\xi \ge 1$ and any integer $T'$, we have
\begin{equation}\label{eq:approx-weak2}
\gamma^{T'} = \left( \gamma^{\xi T'} \right)^{\frac{1}{\xi}} \le \left( \frac{M}{m T'} \right)^{\frac{1}{2\xi}}
\end{equation}
\end{fact}

We will use Fact~\ref{fact:dunno} later on in the proof and we note that Remark~\ref{rem:general-tradeoff} follows by using the more general Fact~\ref{fact:dunno2}. of  Let $L' \triangleq L + mD$ which is the Lipschitz constant of the regularized loss function $g$. We have that for any $i \ge 0$, $\publish (\A \left(\D_i \right) ) \sim \N \left(\mu_i, \sigma^2 \Imat_d \right)$, where it follows by the convergence guarantee of Theorem \ref{thm:sc-sm-gd} that
\begin{align}\label{eq:mu-reg}
\begin{split}
\left\Vert \mu_i - \thetastarr_{i} \right\Vert_2 &\le \gamma^{T} \left\Vert \theta'_0 - \thetastarr_{i} \right\Vert_2
\le \frac{2 L' \gamma^\I \left\Vert \theta'_0 - \thetastarr_{i} \right\Vert_2}{Dmn_i}
\le \frac{2 L'}{mn_i} \cdot \gamma^\I
\end{split}
\end{align}

We also have that for any $i \ge 1$, $\publish( \R_\A \left( \D_{i-1}, u_i,  \theta_{i} \right) ) \sim  \N \left(\mu'_i, \sigma^2 \Imat_d \right)$ where
\begin{equation}\label{eq:mu'-reg}
 \left\Vert \mu'_i - \thetastarr_i \right\Vert_2 \le \frac{4L'}{mn} \cdot i \cdot \gamma^{i^2 \I}
\end{equation}
We use induction on $i$ to prove this claim. Let's focus on the base case $i=1$. We have that
\begin{align*}
 \left\Vert \mu'_1 - \thetastarr_{1} \right\Vert_2 &\le \ \gamma^\I \left\Vert \thetahat_0 - \thetastarr_{1} \right\Vert_2 \\
 &\le \gamma^\I \left( \left\Vert \thetahat_0 - \thetastarr_{0} \right\Vert_2 + \left\Vert \thetastarr_{0} - \thetastarr_{1} \right\Vert_2 \right) \\
 &\le \gamma^\I \left(  \frac{2L'}{mn} \cdot \gamma^\I + \frac{2L'}{mn} \right) \\
 &\le \frac{4L'}{mn} \cdot \gamma^\I
\end{align*}
The first inequality follows from Theorem \ref{thm:sc-sm-gd} and the fact that when running Algorithm~\ref{alg:del_scgd} for the first update $i=1$, the initial point $\theta'_0 = \theta_1 \equiv \thetahat_0$ saved by the training algorithm. The second inequality is a simple triangle inequality, and the third follows from Equation~(\ref{eq:mu-reg}) (noting that $ \thetahat_0 \equiv \mu_0$) and the sensitivity Lemma~\ref{lem:par}. Let's move on to the induction step of the argument. Suppose Equation~(\ref{eq:mu'}) holds for some $i\ge1$. We will show that it holds for $(i+1)$ as well. We have that
\begin{align*}
 \left\Vert \mu'_{i+1} - \thetastarr_{{i+1}} \right\Vert_2 &\le \ \gamma^{(i+1)^2 \I} \left\Vert \thetahat_i - \thetastarr_{{i+1}} \right\Vert_2 \\
 &\le \gamma^{(i+1)^2 \I} \left( \left\Vert \thetahat_i - \thetastarr_{i} \right\Vert_2 + \left\Vert \thetastarr_{i} - \thetastarr_{{i+1}} \right\Vert_2 \right) \\
 &\le \gamma^{(i+1)^2 \I} \left( \frac{4L'}{mn} \cdot i \cdot \gamma^{i^2 \I} + \frac{4L'}{mn} \right) \\
 &\le \frac{4L'}{mn} \cdot (i+1) \cdot \gamma^{(i+1)^2 \I}
\end{align*}
The first inequality follows from Theorem \ref{thm:sc-sm-gd} and the fact that when running Algorithm~\ref{alg:del_scgd} for the $(i+1)$th update, the initial point $\theta'_0 = \theta_{i+1} \equiv \thetahat_{i}$ saved by the previous run of the unlearning algorithm. The second inequality is a simple triangle inequality, and the third follows from the induction assumption for $i$ (noting that $ \thetahat_i \equiv \mu'_i$), the sensitivity Lemma~\ref{lem:par}, and the assumption that $n_i \ge n/2$.

Now that we can apply Equation~\eqref{eq:approx-weak} to Equations~\eqref{eq:mu-reg} and \eqref{eq:mu'-reg} to conclude
\begin{equation}\label{eq:dunno}
\forall \, i \ge 0, \ \left\Vert \mu_i - \thetastarr_{i} \right\Vert_2  \le \frac{4L' \sqrt{M}}{m \sqrt{m \I} n }, \quad \forall \, i \ge 1, \ \left\Vert \mu'_i - \thetastarr_i \right\Vert_2 \le \frac{4L' \sqrt{M}}{m \sqrt{m \I} n }
\end{equation}

We therefore have shown that for any $i \geq 1$,
\[
\publish \left(\A \left(\D_i \right) \right) \sim \N \left(\mu_i, \sigma^2 \Imat_d \right), \quad \publish \left( \R_\A \left( \D_{i-1}, u_i,  \theta_{i} \right) \right) \sim  \N \left(\mu'_i, \sigma^2 \Imat_d \right)
\]
where Equation~\eqref{eq:dunno} implies
\[
\left\Vert \mu_i - \mu'_i  \right\Vert_2 \le \frac{8L' \sqrt{M}}{m \sqrt{m \I} n } \triangleq \Delta
\]
It then follows from Lemma~\ref{lem:zcdp} that $\R_\A$ is a $(\frac{\Delta^2}{2 \sigma^2} + \frac{\Delta}{\sigma}\sqrt{2 \log \left( 1 / \delta \right)}, \delta)$-unlearning algorithm for $\A$, where, with $\sigma$ specified in the theorem statement, we get $(\epsilon,\delta)$-unlearning guarantee.

Now let's focus on the accuracy statement of the theorem. Note, similar to the proof of Theorem~\ref{thm:sc-smooth-guarantees}, the convergence bounds in Equation~\eqref{eq:dunno}, the choice of $\sigma$ in the theorem statement, as well as a Gaussian tail bound (Lemma~\ref{lem:normal-tail}), imply that for any update step $i \ge 0$, with probability at least $1-\beta$,
\begin{equation}\label{eq:theta'-opt-weak}
\left\Vert \thetatilde_i - \thetastarr_i \right\Vert_2  = \bigo \left( \frac{ \sqrt{M} \left(L+mD \right) \sqrt{d \log \left( 1/\delta \right)} \log \left(d/\beta \right) }{\epsilon m \sqrt{m \I} n } \right)
\end{equation}
We therefore have that, using a similar analysis as in the proof of Theorem~\ref{thm:c-smooth-guarantees} (see Equation~\eqref{eq:convex-analysis}), for any update step $i \ge 0$, with probability $1-\beta$,
\begin{align*}
f_{\D_i} ( \thetatilde_i ) - f_{\D_i} ( \thetastar_{i} ) &= \bigo \left( \frac{M^2 \left(L+mD \right)^2 d \log \left( 1/\delta \right) \log^2 \left(d/\beta \right) }{m^3 \epsilon^2  n^2 \I } + m D^2 \right)
\end{align*}
Finally, with the choice of $m$ in the theorem,
\[
f_{\D_i} ( \thetatilde_i ) - f_{\D_i} ( \thetastar_{i} ) = \bigo \left(  \sqrt{ \frac{M L D^3 \sqrt{d \log \left( 1/\delta \right)}}{\epsilon n \sqrt{\I}} } \log^2 \left( d / \beta \right) \right) + \bigo \left( n^{-1} \right) + \bigo \left( n^{-\frac{3}{2}} \right)
\]
\end{proof}

\section{Proof of Theorem~\ref{thm:scgd-distributed}}\label{sec:app-distributed}
%!TEX root = main-full.tex
\begin{proof}[Proof of Theorem~\ref{thm:scgd-distributed}]
We first prove the unlearning guarantee. We note that the boosting of our algorithms (running multiple copies of algorithms and picking the best model for publishing) won't matter in our unlearning bounds. In fact, the unlearning guarantee holds for \emph{any} set $l$ of models learned by the algorithms because they have \emph{all} sufficiently come close to their respective optimizers in each chunk. Hence, until we get to the proof of accuracy statement, we imagine the algorithms are run once. We will see how this boosting will be helpful to recover \emph{high probability} accuracy guarantees from the accuracy bounds of \cite{ZDW12} which are \emph{in expectation}.

Fix a training dataset $\D$ of size $n$ and a non-adaptively chosen update sequence $\U = (u_i)_i$. Similar to our previous proofs, we first recall a few notations (from Definition~\ref{def:newdefs}), as well as some new notations for our proof:
\begin{itemize}
\item $\{ \D_i \}_{i\ge0}$ for the sequence of updated datasets. We use $n_i$ ($\ge n/2$) to denote the size of $\D_i$.
\item $\{ \boldsymbol{\cS}_i = (\cS_{ij})_{j=1}^K \}_{i \ge 0}$ for the sequence of partitioned subsampled datasets.
\item $\{ \boldsymbol{\thetahat}_i = (\thetahat_{ij})_{j=1}^K \}_{i\ge0}$ for the sequence of learned parameters in each partition.
\item $\{ \thetahatiavg \}_{i \ge 0} $ for the sequence of averaged learned parameters: $\thetahatiavg = K^{-1} \sum_{j=1}^K \thetahat_{ij}$.
\item $\{ \thetatilde_i = \publish ( \boldthetahat_i ) = \thetahatiavg + Z_i \}_{i\ge0}$ for the sequence of published parameters.
\item $\{ \thetastar_i \}_{i \ge 0}$ is the sequence of target optimizers: $\thetastar_i \triangleq \argmin_{\theta} f_{\D_i} (\theta)$.
\item $\{ \boldsymbol{\theta}^*_i = (\thetastar_{ij})_{j=1}^K \}_{i \ge 0}$ is the sequence of optimizers for partitions: $\thetastar_{ij} \triangleq \argmin_{\theta} f_{\cS_{ij}} (\theta)$.
\item $\{ \thetastariavg \}_{i \ge 0}$ is the average of optimizers for partitions: $\thetastariavg = K^{-1} \sum_{j=1}^K \thetastar_{ij}$.
\item $\{ s_i \}_{i \ge 1}$ for the sequence of number of affected data points in the whole dataset, i.e., $s_i$ shows how many points differ between $\boldcS_i$ and $\boldcS_{i-1}$. We will also make use of notation $s_{ij}$ which shows how many points differ between $\cS_{ij}$ and $\cS_{i-1,j}$. Note that $s_i = \sum_{j=1}^K s_{ij}$.
\end{itemize}

\begin{fact}\label{fact:chernoff}
Let $\tilde{s}_i \triangleq \max_{l \le i} s_l$. We have by Lemma~\ref{samp2} that for any $i$, with probability at least $1-\delta/2$ over the sampling randomness up to round $i$, $\tilde{s}_i \le \frac{10B}{n} \log\left( 2i / \delta \right)$. We condition on this high probability event throughout the proof.
\end{fact}

\begin{fact}
We also work with general $K$ and $B$ for now and eventually we use the ones stated in the theorem. We note that for general $K$ and $B$, we can write
\[
T \ge \frac{Kn^2 \I}{B^2} + \frac{\log \left( Dm L^{-1} B \left( 1 + 10 \log\left( 2 / \delta \right) \right) \right)}{\log \left( 1 / \gamma \right)}
\]
and
\[
T_i = 10\log\left( 2i / \delta \right) \left(\I +  \frac{B^2}{K n^2} \cdot \frac{\log \left( 1 + 10 i \log\left( 2i / \delta \right) \right)}{ \log \left( 1/ \gamma \right)}\right)
\]
Let $T_i'$ be the number of iterations in \emph{affected partitions} on round $i$. We have that with probability at least $1-\delta/2$, by Fact~\ref{fact:chernoff},
\begin{equation}\label{eq:runtimechunk}
T_i' \ge \frac{Kn}{Bs_i} T_i \ge \frac{Kn^2}{10 B^2 \log\left( 2i / \delta \right)} T_i  \ge \frac{\log \left( 1 + 10 i \log\left( 2i / \delta \right) \right)}{ \log \left( 1/ \gamma \right)} + \frac{Kn^2 \I}{B^2}
\end{equation}
\end{fact}

\begin{fact}
We have that $B \ge n$, and $Kn^2 \ge B^2$ (note these are justified by the setting of these parameters in theorem statement). We will use these later on in the proof.
\end{fact}

For every $i \ge 1$, let $\boldsymbol{\cS}_i'$ be the partitioned dataset we would have had we retrained (using our learning algorithm $\A$) on dataset $\D_i$, and note that by Lemma~\ref{samp}, $\boldsymbol{\cS}_i'$ and $\boldsymbol{\cS}_i $ are distributed identically. To apply Lemma~\ref{samp} we have used the fact that $\U$ is a non-adaptive sequence of updates selected independently of any internal randomness of $\R_\A$. Now let $\mathcal{C}_i$ be a coupling of the pair $(\boldsymbol{\cS}_i', \boldsymbol{\cS}_i)$ such that $\boldsymbol{\cS}_i' = \boldsymbol{\cS}_i$ with probability one. Throughout the proof when we condition on any of $\boldsymbol{\cS}_i'$ or $\boldsymbol{\cS}_i$ being drawn from their distribution, we will think of these datasets being drawn from the coupling $\cC_i$ so that we are always guaranteed $\boldsymbol{\cS}_i' = \boldsymbol{\cS}_i$. 
Let's start proving the unlearning guarantees. For any $i \ge 0$, conditioned on the draw of $\boldsymbol{\cS}_i'$, we have that $\publish (\A \left(\D_i \right) ) \sim \N \left(\mu_i, \sigma^2 \Imat_d \right)$, where $\mu_i = K^{-1} \sum_{j=1}^K \mu_{ij}$ and that it follows by the convergence guarantee of Theorem \ref{thm:sc-sm-gd} that, for all partitions $j$,
\begin{align}\label{eq:mu-distributed}
\begin{split}
\left\Vert \mu_{ij} - \thetastar_{ij}  \right\Vert_2 &\le \gamma^{T} \left\Vert \theta'_0 - \thetastar_{ij} \right\Vert_2
\le \frac{4L \gamma^{\frac{Kn^2\I}{B^2}} \left\Vert \theta'_{0} - \thetastar_{ij} \right\Vert_2}{DmB \left( 1 + 10 \log\left( 2 / \delta \right) \right)}
\le \frac{4L}{mB  \left( 1 + 10 \log\left( 2 / \delta \right) \right)} \cdot \gamma^{\frac{Kn^2\I}{B^2}}
\end{split}
\end{align}
We also have that for any update step $i \ge 1$, with probability at least $1-\delta/2$ over the randomness up to step $i$ (draws of all $\boldcS_l$ for all $l\le i$), $\publish \left( \R_\A \left( \boldcS_{i-1}, u_i,  \boldtheta_{i} \right) \right) \sim  \N \left(\mu'_i, \sigma^2 \Imat_d \right)$, where we first observe that $\mu'_i = \thetahatiavg = K^{-1}  \sum_{j=1}^K \thetahat_{ij}$, and furthermore,
\begin{equation}\label{eq:mu'-distributed}
\forall \, j; \ \left\Vert \thetahat_{ij} -   \thetastar_{ij} \right\Vert_2 \le  \frac{4LK \left(K^{-1} + \sum_{l\le i} s_{lj} \right) }{mB  \left( 1 + 10i \log\left( 2i / \delta \right) \right)} \cdot \frac{\gamma^{\frac{Kn^2\I}{B^2}}}{1 -  \gamma^{\frac{Kn^2\I}{B^2}}}
\end{equation}
We use induction on $i$ to prove this claim. Let's focus on the base case $i=1$. For any partition $j$ such that $s_{1j} = 0$, because the update algorithm doesn't make any updates, we have
\begin{align*}
\left\Vert \thetahat_{1j} - \thetastar_{1j} \right\Vert_2 &= \left\Vert \thetahat_{0j} - \thetastar_{0j} \right\Vert_2 \\
&\le \frac{4L}{mB \left( 1 + 10 \log\left( 2 / \delta \right) \right)} \cdot \gamma^{\frac{Kn^2\I}{B^2}} \\
&\le \frac{4LK \left(K^{-1} + s_{1j} \right) }{mB\left( 1 + 10 \log\left( 2 / \delta \right) \right)} \cdot \frac{\gamma^{\frac{Kn^2\I}{B^2}}}{1 -  \gamma^{\frac{Kn^2\I}{B^2}}}
\end{align*}
because note that $\thetahat_{0j} \equiv \mu_{0j}$ and therefore, we can use Equation~\eqref{eq:mu-distributed} for $i=0$. For any partition $j$ such that $s_{1j} \neq 0$, the update algorithm makes update, and in particular runs for $T'_1$ iterations. We therefore have that
\begin{align*}
\left\Vert \thetahat_{1j} - \thetastar_{1j} \right\Vert_2 &\le \gamma^{T'_1} \left\Vert \thetahat_{0j} - \thetastar_{1j} \right\Vert_2 \\
& \le \gamma^{T'_1} \left(  \left\Vert \thetahat_{0j} - \thetastar_{0j} \right\Vert_2 +  \left\Vert \thetastar_{0j} - \thetastar_{1j} \right\Vert_2 \right) \\
&\le \gamma^{T'_1} \left( \frac{4L}{mB} \cdot  \frac{\gamma^{\frac{Kn^2\I}{B^2}}}{1 -  \gamma^{\frac{Kn^2\I}{B^2}}} + \frac{4 L K s_{1j}}{mB} \right) \\
&\le \frac{4LK \left(K^{-1} + s_{1j} \right) }{mB \left( 1 + 10 \log\left( 2 / \delta \right) \right)} \cdot \frac{\gamma^{\frac{Kn^2\I}{B^2}}}{1 -  \gamma^{\frac{Kn^2\I}{B^2}}}
\end{align*}
The first inequality follows from the convergence guarantee of Theorem \ref{thm:sc-sm-gd} and the fact that on round $i+1$ of update, the gradient descent of chunk $j$ is initialized at $\thetahat_{0j}$. The second inequality is a triangle inequality and the third follows from Equation~\eqref{eq:mu-distributed} for $i=0$ (note $\thetahat_{0j} \equiv \mu_{0j}$), and the sensitivity Lemma~\ref{lem:par} (note that we apply this Lemma $2s_{1j}$ times and that the size of each chunk is $B/K$). The last inequality follows from Equation~\eqref{eq:runtimechunk}. Now let's focus on the induction step of the argument. Suppose Equation~\eqref{eq:mu'-distributed} holds for some $i \ge 1$. We will show that it holds for $i+1$ as well. For any partition $j$ such that $s_{i+1,j} =0$, we have $\Vert \thetahat_{i+1,j} - \thetastar_{i+1,j} \Vert_2 = \Vert \thetahat_{i,j} - \thetastar_{i,j} \Vert_2$ and the claim holds by induction assumption. Now suppose $s_{i+1,j} \neq 0$ which implies the update algorithm runs $T'_{i+1}$ iterations of gradient descent on chunk $j$. We therefore have that, similar to how we proceed for $i=1$ case above,
\begin{align*}
\left\Vert \thetahat_{i+1,j} - \thetastar_{i+1,j} \right\Vert_2 &\le \gamma^{T'_{i+1}} \left\Vert \thetahat_{ij} - \thetastar_{i+1,j} \right\Vert_2 \\
&\le \gamma^{T'_{i+1}} \left( \left\Vert \thetahat_{ij} - \thetastar_{ij} \right\Vert_2 + \left\Vert \thetastar_{ij} - \thetastar_{i+1,j} \right\Vert_2 \right) \\
&\le \gamma^{T'_{i+1}} \left( \frac{4LK \left(K^{-1} + \sum_{l\le i} s_{lj} \right) }{mB} \cdot \frac{\gamma^{\frac{Kn^2\I}{B^2}}}{1 -  \gamma^{\frac{Kn^2\I}{B^2}}} + \frac{4 L K s_{i+1,j}}{mB} \right) \\
&\le \frac{4LK \left(K^{-1} + \sum_{l\le i+1} s_{lj} \right) }{mB \left( 1 + 10 (i+1) \log\left( 2(i+1) / \delta \right) \right)} \cdot \frac{\gamma^{\frac{Kn^2\I}{B^2}}}{1 -  \gamma^{\frac{Kn^2\I}{B^2}}}
\end{align*}
where the third inequality follows from induction assumption for $i$ and applying the sensitivity Lemma~\ref{lem:par} $2s_{i+1,j}$ times, and the last inequality follows from Equation~\eqref{eq:runtimechunk}. This completes the induction proof. Now we can use Equations~\eqref{eq:mu-distributed} and \eqref{eq:mu'-distributed} to conclude that for all $i \ge 0$,
\begin{equation}\label{eq:mu-dist}
\left\Vert \mu_i - \thetastariavg \right\Vert \le \frac{1}{K} \sum_{j=1}^K \left\Vert \mu_{ij} - \thetastar_{ij}  \right\Vert_2 \le \frac{4L}{mB  \left( 1 + 10 \log\left( 2 / \delta \right) \right)} \cdot \gamma^{\frac{Kn^2\I}{B^2}} \le \frac{4L}{mn} \cdot \gamma^{\frac{Kn^2\I}{B^2}}
\end{equation}
where we use the fact that $B \ge n$. And with probability at least $1-\delta/2$, for all $i \ge 1$,
\begin{align}\label{eq:mu'-dist}
\begin{split}
\left\Vert \mu'_i - \thetastariavg \right\Vert &\le \frac{1}{K} \sum_{j=1}^K \left\Vert \thetahat_{ij} - \thetastar_{ij}  \right\Vert_2 \\
&\le \frac{4L }{mB  \left( 1 + 10 i \log\left( 2i / \delta \right) \right)} \cdot \frac{\gamma^{\frac{Kn^2\I}{B^2}}}{1 -  \gamma^{\frac{Kn^2\I}{B^2}}} \sum_{j=1}^K \left(K^{-1} + \sum_{l\le i} s_{lj} \right) \\
&= \frac{4L }{mB  \left( 1 + 10 i \log\left( 2i / \delta \right) \right)} \cdot \frac{\gamma^{\frac{Kn^2\I}{B^2}}}{1 -  \gamma^{\frac{Kn^2\I}{B^2}}} \left(1 + \sum_{l\le i} s_{l} \right) \quad (\text{because $\sum_{j} s_{lj} = s_l$})\\
&\le \frac{4L }{mB  \left( 1 + 10 i \log\left( 2i / \delta \right) \right)} \cdot \frac{\gamma^{\frac{Kn^2\I}{B^2}}}{1 -  \gamma^{\frac{Kn^2\I}{B^2}}} \left(1 +  i \tilde{s}_i \right) \quad  (\text{recall $\tilde{s}_i = \max_{l \le i} s_l $}) \\
&\le \frac{4L }{mB  \left( 1 + 10 i \log\left( 2i / \delta \right) \right)} \cdot \frac{\gamma^{\frac{Kn^2\I}{B^2}}}{1 -  \gamma^{\frac{Kn^2\I}{B^2}}}  \left(1 + i\frac{10B}{n} \log \left( 2i / \delta \right) \right) \\
&= \frac{4L }{mn  \left( 1 + 10 i \log\left( 2i / \delta \right) \right)} \cdot \frac{\gamma^{\frac{Kn^2\I}{B^2}}}{1 -  \gamma^{\frac{Kn^2\I}{B^2}}}  \left(\frac{n}{B} + 10i \log \left( 2i / \delta \right) \right) \\
&\le \frac{4L }{mn  \left( 1 + 10 i \log\left( 2i / \delta \right) \right)} \cdot \frac{\gamma^{\frac{Kn^2\I}{B^2}}}{1 -  \gamma^{\frac{Kn^2\I}{B^2}}}  \left(1 + 10i \log \left( 2i / \delta \right) \right) \quad (\text{because $B \ge n$}) \\
&= \frac{4L}{mn} \cdot \frac{\gamma^{\frac{Kn^2\I}{B^2}}}{1 -  \gamma^{\frac{Kn^2\I}{B^2}}}
\end{split}
\end{align}
implying that for any $i \ge 1$, conditioned on the event that $\left\{ \tilde{s}_i \le 10B n^{-1} \log \left(2i/\delta\right) \right\}$ which holds with probability at least $1-\delta/2$ (by Fact~\ref{fact:chernoff}),
\begin{equation}
\left\Vert \mu_i - \mu'_i \right\Vert_2 \le \frac{8L}{mn} \cdot \frac{\gamma^{\frac{Kn^2\I}{B^2}}}{1 -  \gamma^{\frac{Kn^2\I}{B^2}}} \triangleq \Delta
\end{equation}
It then follows from Lemma~\ref{lem:zcdp}, as well as the choice of
\begin{equation*}
\sigma = \frac{ 4 \sqrt{2} L \gamma^{\frac{Kn^2\I}{B^2}}}{ mn \left( 1 - \gamma^{\frac{Kn^2\I}{B^2}} \right) \left(\sqrt{\log \left( 2/\delta \right) + \epsilon} - \sqrt{\log \left( 2/\delta \right)} \right)}
\end{equation*}
in the theorem statement, that for any $i \geq 1$, with probability at least $1-\delta/2$,
$
\publish \left(\A \left(\D_i \right) \right) \overset{\epsilon,\delta/2}{\approx} \publish \left( \R_\A \left( \boldcS_{i-1}, u_i,  \boldtheta_{i} \right) \right)
$.
Now we can apply Lemma~\ref{lem:high-prob-zcdp} to conclude that for any $i \geq 1$,
\begin{equation*}
\publish \left(\A \left(\D_i \right) \right) \overset{\epsilon,\delta}{\approx} \publish \left( \R_\A \left( \boldcS_{i-1}, u_i,  \boldtheta_{i} \right) \right)
\end{equation*}
And this shows $\R_\A$ is an $(\epsilon, \delta)$-unlearning algorithm for $\A$, as desired.

Now let's prove the accuracy statement of the theorem for which we will make use of Equations (\ref{eq:mu-dist}) and (\ref{eq:mu'-dist}) (which holds with probability $1-\delta$). Recall that $\thetahat_{0,\text{avg}} \equiv \mu_0$ and $\thetahatiavg \equiv \mu'_i$ for $i \ge 1$. We first state the accuracy in expectation and then finally will turn those into high probability accuracy guarantees. First, we have that by a simple application of Cauchy-Schwarz inequality,
\begin{align}\label{eq:cs}
\begin{split}
\mathbb{E} \left\Vert \thetahatiavg - \thetastar_i \right\Vert_2^2 &=  \mathbb{E} \left\Vert \thetahatiavg - \thetastariavg + \thetastariavg - \thetastar_i \right\Vert_2^2 \\
&=\mathbb{E} \left\Vert \thetahatiavg - \thetastariavg \right\Vert_2^2  +  \mathbb{E} \left\Vert \thetastariavg - \thetastar_i \right\Vert_2^2 + 2 \mathbb{E} \left( \thetahatiavg - \thetastariavg \right)^\top \left( \thetastariavg - \thetastar_i  \right)  \\
&\le \mathbb{E} \left\Vert \thetahatiavg - \thetastariavg \right\Vert_2^2  +  \mathbb{E} \left\Vert \thetastariavg - \thetastar_i \right\Vert_2^2 + \sqrt{\mathbb{E} \left\Vert \thetahatiavg - \thetastariavg \right\Vert_2^2  \mathbb{E} \left\Vert \thetastariavg - \thetastar_i \right\Vert_2^2 }
\end{split}
\end{align}
but, by an application of law of total expectation (to turn the high probability guarantees into bounds in expectation),
\begin{equation}\label{eq:our-bound}
\mathbb{E} \left\Vert \thetahatiavg - \thetastariavg \right\Vert_2^2 \le \frac{16L^2}{m^2n^2} \cdot \frac{\gamma^{\frac{2Kn^2\I}{B^2}}}{\left(1 -  \gamma^{\frac{Kn^2\I}{B^2}} \right)^2} + \delta D^2
\end{equation}
and we also know by Theorem~\ref{thm:zdw} that, for some constant $c$, and for the choice of $K = \sqrt{B}$,
\begin{align}\label{eq:zdw-bound}
\begin{split}
\mathbb{E}  \left\Vert \thetastariavg - \thetastar_i \right\Vert_2^2 &\le \frac{2L^2}{m^2 B} + \frac{c L^2 K^2}{m^4 B^2} \left(H^2 \log d + \frac{L^2G^2}{m^2} \right) + \bigo \left(\frac{K}{B^2} \right) + \bigo \left(\frac{K^3}{B^3} \right) \\
&= \frac{2L^2}{m^2 B} + \frac{c L^2}{m^4 B} \left(H^2 \log d + \frac{L^2G^2}{m^2} \right) + \bigo \left( B^{-\frac{3}{2} }\right) \\
&= \frac{1}{B} \left( \frac{2L^2}{m^2} + \frac{c L^2}{m^4} \left(H^2 \log d + \frac{L^2G^2}{m^2} \right) \right) + \bigo \left( B^{-\frac{3}{2} }\right)
\end{split}
\end{align}
Putting together Equations~\eqref{eq:cs} and \eqref{eq:our-bound} (with $K = \sqrt{B}$) and \eqref{eq:zdw-bound}, and noting that for $\delta = \bigo ( B^{-1})$ and $B \ge n$ we have $\sqrt{\mathbb{E} \Vert \thetahatiavg - \thetastariavg \Vert_2^2  \cdot \mathbb{E} \Vert \thetastariavg - \thetastar_i \Vert_2^2 } = \bigo(\log d / B) $, and hiding all constants under the $\bigo$ notation, we have
\begin{equation*}
\mathbb{E} \left\Vert \thetahatiavg - \thetastar_i \right\Vert_2^2 =  \bigo \left( \frac{\gamma^{\frac{n^2\I}{B \sqrt{B}}}}{n^2 \left(1 -  \gamma^{\frac{n^2\I}{B \sqrt{B}}} \right)^2 }\right) +  \bigo \left( \frac{\log d}{B} \right) + \bigo \left( \frac{1}{B^{\frac{3}{2}}} \right)
\end{equation*}
Now by Lemma~\ref{lem:exp-to-highprob}, we have that by running the algorithm for $C = \log \left( 2/\beta \right) / \log 2$ times and picking the best model with smallest loss (note by strong convexity, the smaller the loss of a model is, the closer the model parameter is to the optimizer. Also for notational convenience, we still use $\thetahatiavg$ for the best model), with probability at least $1-\beta/2$,
\begin{equation}\label{eq:withoutnoise}
\left\Vert \thetahatiavg - \thetastar_i \right\Vert_2^2 =  \bigo \left( \frac{\gamma^{\frac{n^2\I}{B \sqrt{B}}}}{n^2 \left(1 -  \gamma^{\frac{n^2\I}{B \sqrt{B}}} \right)^2 }\right) +  \bigo \left( \frac{\log d}{B} \right) + \bigo \left( \frac{1}{B^{\frac{3}{2}}} \right)
\end{equation}
Recall that at any given round $i \ge 0$, the published model $\thetatilde_i = \thetahatiavg+ Z_i$. We therefore have that by Equation~\eqref{eq:withoutnoise}, a Gaussian tail bound (Lemma~\ref{lem:normal-tail}), choice of $\sigma$ in the theorem statement, and the fact that for $\epsilon = \bigo \left(\log \left( 1/\delta \right) \right)$, we have $\sqrt{\log \left( 1/\delta \right) + \epsilon} - \sqrt{\log \left( 1/\delta \right)} = \Omega ( \epsilon/ \sqrt{\log \left( 1/\delta \right)} )$, with probability at least $1-\beta$,
\begin{equation}
\left\Vert \thetatilde_i - \thetastar_i \right\Vert_2^2 = \bigo \left( \frac{L^2 \gamma^{\frac{2n^2\I}{B \sqrt{B}}} d \log \left( 1/\delta \right) \log^2 \left(d/\beta \right)  }{m^2 \epsilon^2 n^2 \left(1 -  \gamma^{\frac{Kn^2\I}{B^2}} \right)^2 }\right) +  \bigo \left( \frac{\log d}{B} \right) + \bigo \left( \frac{1}{B^{\frac{3}{2}}} \right)
\end{equation}
Note that $(1-\gamma^a)^{-1} \le (1 - \gamma)^{-1}$ for any $a \ge 1$ (in our case $a = \frac{Kn^2\I}{B^2} \ge 1$). The proof is complete by the choice of $B = n^\xi$ and $M$-smoothness of $f$:
\begin{equation*}
f_{\D_i} (\thetatilde_i) - f_{\D_i} (\thetastar_i) \le \frac{M}{2} \left\Vert \thetatilde_i - \thetastar_i \right\Vert_2^2
\end{equation*}
\end{proof}

}

\end{document}